%% file: paper_main.tex
\documentclass[9pt, twocolumn]{article}
\usepackage[left=1.8cm,right=1.8cm,top=3cm,bottom=3cm]{geometry}
\usepackage{graphicx}
\usepackage{cite}
\usepackage[cmex10]{amsmath}

\usepackage{amssymb,amsthm}
\newtheorem{theorem}{Theorem}

\usepackage{latexsym} 
\usepackage{tikz}
\usetikzlibrary{decorations.markings,shapes,arrows,positioning}

\usepackage{booktabs}
\usepackage{url}

\begin{document}

\title{Segmentation of Three-dimensional Images with Parametric Active Surfaces and Topology Changes}
\author{Heike~Benninghoff% <-this % stops a space      
\thanks{Deutsches Zentrum f\"ur Luft- und Raumfahrt (DLR), 82234 We\ss ling, Germany, email: heike.benninghoff@dlr.de.} % <-this % stops a space
~and~Harald~Garcke% <-this % stops a space
\thanks{Fakult\"at f\"ur Mathematik, Universit\"at Regensburg, 93040 Regensburg, Germany, email: harald.garcke@ur.de.}}% <-this % stops a space

\date{}
\maketitle

\begin{abstract}
In this paper, we introduce a novel parametric method for segmentation of three-dimensional images. We consider a piecewise constant version of the Mumford-Shah and the Chan-Vese functionals and perform a region-based segmentation of 3D image data. An evolution law is derived from energy minimization problems which push the surfaces to the boundaries of 3D objects in the image.  We propose a parametric scheme which describes the evolution of parametric surfaces. An efficient finite element scheme is proposed for a numerical approximation of the evolution equations. Since standard parametric methods cannot handle topology changes automatically, an efficient method is presented to detect, identify and perform changes in the topology of the surfaces. One main focus of this paper are the algorithmic details to handle topology changes like splitting and merging of surfaces and change of the genus of a surface. Different artificial images are studied to demonstrate the ability to detect the different types of topology changes. Finally, the parametric method is applied to segmentation of medical 3D images. 

\end{abstract}

\input{1_introduction.tex}
\input{2_segmentation.tex}
\input{3_numerics.tex}
\input{4_results.tex}
\input{5_conclusion.tex}

\section*{Acknowledgements}
The authors would like to thank Prof. Dr. Christian Stroszczynski, Department of Radiology of University Hospital Regensburg, for providing computed tomography images which have been used in Figure~\ref{fig:cd4_part2_01_0000}-\ref{fig:cd4_part2_01_0900}.

\bibliographystyle{plain} 
\bibliography{literatur}   % name your BibTeX data base

\end{document}

%% file: 1_introduction.tex
\section{Introduction}
\label{intro}
One major challenge in image processing is the autonomous detection of objects in images and the segmentation of the objects from each other and from their environment. 

A very popular approach for image segmentation is the active contour method \cite{Kass88}, \cite{Cohen91}. In the case of classical two-dimensional images, one or more curves, called contours, evolve in the two-dimensional image domain and stop locally at edges or region boundaries. The motion is described by evolution equations which aim to minimize a certain energy functional. The energies typically contain length terms to control the smoothness of the contours (internal energies) and terms which push the contours to the desired region boundaries or to edges in the image (external energies). 

Two main classes of approaches can be distinguished: The first class are edge-based methods where regions are identified by their boundaries where the image intensity function rapidly changes \cite{Kass88}, \cite{Malladi95}, \cite{Caselles97}. The second class are region-based methods where the regions are characterized by the mean gray value or mean color, or by the texture or some other grouping \cite{Mumford89}, \cite{Ronfard94}, \cite{Chan01}, \cite{Tsai01}.  

Region-based active contours methods can also be applied on images with so-called weak edges, i.e. edges with only small changes in the image intensity function \cite{Chan01}, and on images which contain regions which are groups of smaller objects \cite{Aubert06}. Furthermore, the method is less sensitive to noise. If images with high noise have to be segmented, gradient-based approaches may get trapped at locations where the noise is high and the contours may not detect the \emph{real} objects in the image.

In this paper, we study volumetric, i.e. three-dimen\-sional, images given by a scalar or vector-valued image function $u_0:\Omega \rightarrow \mathbb{R}^{(d)}$, where $\Omega \subset \mathbb{R}^3$ is an open and bounded image domain. Real images are often defined on a set of $N_x \times N_y \times N_z$ voxels (=volume pixels), where $u_0$ is locally constant on each voxel. 3D images may be reconstructed by certain 3D imaging procedures like computed tomography (CT) or magnetic resonance imaging (MRI), cf. \cite{Udupa99}, \cite{Scherzer2009}.  

3D image segmentation aims at dividing a given image in connected regions, representing 3D objects in the image or their environment in the image domain $\Omega$. For 3D image segmentation, the boundaries of the regions or objects have to be detected. These boundaries can be represented by a set of two-dimensional surfaces. 

The active contours concept \cite{Kass88} can be extended to the three-dimensional case. For 3D images, we can consider time-dependent two-dimensional surfaces (\emph{active surfaces}) and evolution laws for the surfaces which attract them to region boundaries. In particular, we will study extensions of the Mum\-ford-Shah model \cite{Mumford89} and the Chan-Vese model \cite{Chan01} to the three-dimensional case. 

The articles \cite{Cohen91}, \cite{Cohen93} belong to the first works, where the active contours model \cite{Kass88} is extended to volumetric image data. There, an analytical framework is introduced for 3D deformable surfaces. For practical computations however, the authors suggest to replace a given 3D image by a sequence of 2D images and to apply the 2D active contour model on each single image, followed by a 3D reconstruction of the surface. 

The  geodesic active contours model \cite{Caselles97} is a popular edge-based method. An extension of the geodesic active contours model to 3D image segmentation is proposed in \cite{Caselles97b}. The level set method \cite{OsherSethian88} is used to describe the surface implicitly. 

The level set method is also used in \cite{Yezzi1997}, where 2D and 3D active contour models are presented and applied on medical images. However, practical results are only shown for 2D images. 
A detailed literature study on 3D brain cortex segmentation is given in \cite{Li2005}.  A review on segmentation of medical X-ray computed tomography  and magnetic resonance images is presented in \cite{Sharma2010}. 

In \cite{Mille2009}, a combination of edge-based and region-based segmentation methods is proposed. Both explicit (triangulated surfaces) and implicit (level set) methods  are implemented. For explicit methods a constant global topology is assumed. For images which require topology changes, only the level set method is applied. 

The Chan-Vese model \cite{Chan01} is used for 2D and 3D medical applications in \cite{Rousseau2009} to perform heart segmentation using an iterative version of the Chan-Vese algorithm. The level set method with a finite difference scheme is used to solve the segmentation problem numerically. The level set method is also applied in \cite{Ardon2005} and \cite{Shen2009} for active surfaces. Applications using 3D medical data (i.a. lung and heart segmentation) are considered. In \cite{Mikula11b}, the level-set method is applied for 3D cell membrane segmentation. An approach for tracking cells in 4D images (3D data + time) has been developed in \cite{Mikula14b}.

Also some parametric approaches for surface evolution exist in literature: In \cite{Brakke1992}, a program called "The Surface Evolver" is presented which computes evolving triangulated surfaces, where the evolution is driven by energy minimization problems with possible constraints. The program is able to perform topology changes like splitting if this is instructed by the user. 

Another explicit method for evolving surfaces by using triangulated surfaces is proposed in \cite{Brochu09}. Techniques are introduced for mesh quality improvement and for topology changes. Splitting is done by a combination of removing degenerate elements and a certain mesh separation method based on duplication and separation of nodes. Surfaces are thus split if they become locally too thin. Merging is detected by searching for edges which are too close. 

Finite element approaches for surface evolutions are pursued in \cite{Baensch05}. There, the authors do not consider image segmentation applications, but surface diffusion. Several mesh quality routines like mesh regularization (keeping all angles of simplices at a node of the same size), time step control, refine/coarsening routines and angle width control are proposed. 

In this paper, we present a novel parametric approach for 3D image segmentation. We present a scheme for image segmentation describing the evolution of parametric surfaces. For the numerical approximation, the smooth surfaces are replaced by triangulated surfaces. We make use of a parametric finite element scheme based on \cite{BGN08c}. There, a scheme is proposed for surface diffusion, (inverse) mean curvature flow and non-linear flows. We use and apply this scheme to image segmentation with multiple phases and regions.  

Our method also allows for topology changes which have not been addressed in \cite{BGN08c}. We efficiently detect topology changes and perform modifications of the surface triangulations. In \cite{Benninghoff2014a}, we considered segmentation of two-dimensional images, and used and extended a method to detect topology changes \cite{MikulaUrban12} to handle a variety of topology changes of curves. Topology changes involving surfaces are more complex compared to topology changes involving curves. For example, if a curve splits up in two subcurves, the discretization has to be modified only at two points, see \cite{Benninghoff2014a} for details. If a surface is split up in two subsurfaces, many triangles are located in a small volume. The pure detection of such a splitting is quite simple; the idea of an auxiliary background grid \cite{MikulaUrban12} as used for curves can be extended to topology changes of surfaces. However, the modifications of the surface triangulation are not as straight-forward as for curves. In case of splitting, we will delete the involved triangles near the splitting point resulting in two surfaces with intermediate holes. Then, we will close the intermediate holes by creating new triangles. Apart from splitting and merging, further topology changes can occur for surfaces: An increase or decrease of the genus of a surface can also occur, for example when a sphere evolves to a torus or vice versa. In summary, the execution of topology changes is the main additional challenge of 3D image segmentation with parametric surfaces. Therefore, one main focus of this paper is the detection, identification and execution of topology changes.  

The remaining part of this paper is structured as follows. In Section~\ref{sec:segmentation}, we present a region-based active surface model where we extend the  Mumford-Shah model \cite{Mumford89} and the Chan-Vese model \cite{Chan01} to 3D images. We present an efficient parametric scheme for the evolving surfaces. Also multiple phases can be handled. The main part of this paper is the numerical approximation of our scheme and the handling of topology changes which is described in Section~\ref{sec:numerics}. A finite element scheme is presented, a corresponding linear equation is derived and some computational details are given including mesh quality aspects and time step control. The detection, identification and execution of topology changes is described in detail including a description how to modify the triangulations after a topology change has been detected. In Section~\ref{sec:results}, we present results from segmentation of artificial test images and real medical images. We demonstrate the different topology changes which can occur during the evolution of surfaces. A final conclusion is drawn in Section~\ref{sec:conclusion}. 

%% file: 2_segmentation.tex
\section{Segmentation of three-dimensional images}
\label{sec:segmentation}

\subsection{Region-based Active Surfaces}
We perform image segmentation by active surfaces, the surface-analogue to active contours  \cite{Kass88}, \cite{Cohen91}. The idea of active surfaces is to let surfaces $\Gamma(t)$, $t \in [0,T]$, evolve in time such that a certain energy functional is minimized. From the minimization problem, one can derive an evolution law such that the surfaces of $\Gamma(t)$ are attracted to the region boundaries in the given image. 

Let $\Omega \subset \mathbb{R}^3$ be open and bounded. We first consider a scalar image function $u_0: \Omega \rightarrow \mathbb{R}$.

In this paper, we restrict on region-based methods for segmentation of 3D images because of advantages of region-based approaches compared to edge-based approaches (cf. Section~\ref{intro}). 

The Mumford-Shah \cite{Mumford89} method for 3D images aims at finding a set of two-dimensional surfaces $\Gamma = \Gamma_1 \cup \ldots \cup  \Gamma_{N_C}$ and a piecewise smooth function $u:\Omega \rightarrow \mathbb{R}$ with possible discontinuities across $\Gamma$ approximating the original image $u_0$. The energy to be minimized is 
\begin{equation}
E^{\mathrm{MS}}(u, \Gamma) = \sigma |\Gamma| + \int_{\Omega\setminus \Gamma} \|\nabla u\|^2 \,\mathrm{d}x + \lambda \int_\Omega (u_0-u)^2 \,\mathrm{d}x,
\label{eq:mumford_shah}
\end{equation}
where $\sigma, \lambda > 0$ are weighting parameters and $|\Gamma|$ denotes the total area of the surfaces belonging to $\Gamma$. (For non-smooth surfaces, we identify $|\Gamma|$ with the two-dimensional Hausdorff measure of $\Gamma\subset \mathbb{R}^3$.) 

The first term in \eqref{eq:mumford_shah} penalizes the area of the surfaces, the second term does not allow $u$ to change much in $\Omega \setminus\Gamma$, and the third term requests that $u$ is a good approximation of $u_0$. 

We first consider two-phase image segmentation, we consider one closed, orientable surface $\Gamma$ separating two disjoint regions $\Omega_1$ and $\Omega_2$ such that $\Omega = \Omega_1 \cup \Gamma \cup \Omega_2$. We assume that $\Gamma$ is oriented by a unit normal vector field $\vec \nu$ pointing from $\Omega_2$ to $\Omega_1$. 

Furthermore, we consider a piecewise constant version of the Mumford-Shah functional: We search for a surface $\Gamma$ and for an approximation $u: \Omega \rightarrow \mathbb{R}$ of $u_0$ which is piecewise constant in each region, i.e. $u_{|\Omega_k} = c_k$, $k=1,2$, such that  
\begin{align}
&E(\Gamma,c_1,c_2) =  \nonumber\\
& =\sigma |\Gamma| + \lambda \left(\int_{\Omega_1} (u_0-c_1)^2\,\mathrm{d}x +  \int_{\Omega_2} (u_0-c_2)^2\,\mathrm{d}x\right)
\label{eq:mumford_shah_const_3D}
\end{align}
is minimized.

Similarly, the Chan-Vese functional \cite{Chan01} can be extended to 3D images. The energy to be minimized is
\begin{align}
E(\Gamma, c_1, c_2) =& \sigma |\Gamma|  +  \mu \int_{\Omega_1} 1 \,\mathrm{d}x + \lambda_1  \int_{\Omega_1} (u_0-c_1)^2\,\mathrm{d}x \nonumber \\
&+ \lambda_2 \int_{\Omega_2}(u_0-c_2)^2\,\mathrm{d}x,
\label{eq:functional_chanvese}
\end{align}
where $\sigma, \lambda_1, \lambda_2 > 0$, $\mu \geq 0$ are weighting parameters. For $\mu=0$ and $\lambda_1 = \lambda_2 = \lambda$, this is the functional \eqref{eq:mumford_shah_const_3D}.

The energy defined in \eqref{eq:mumford_shah_const_3D} depends on the surface $\Gamma$ and on the image approximation $u$ given by the coefficients $c_1$, $c_2$. For minimizing \eqref{eq:mumford_shah_const_3D}, we perform a two-step approach:

First, we fix the surface $\Gamma$ and consider variations in the coefficients $c_1$, $c_2$. Using the theory of calculus of variations we obtain the mean of the image function  in $\Omega_k$ for $c_k$, $k=1,2$:
\begin{equation}
c_k = \frac{\int_{\Omega_k} u_0 \,\mathrm{d}x}{\int_{\Omega_k} 1 \,\mathrm{d}x}.
\label{eq:c_k}
\end{equation}
Then, we fix $c_1$ and $c_2$ and consider small variations of the surface $\Gamma$ by smooth surfaces $\Gamma(t) \subset \Omega$, $t \in (-\epsilon,\epsilon)$, with $\Gamma(0)=\Gamma$. Let $\Omega_1(t)$ and $\Omega_2(t)$ be the regions separated by $\Gamma(t)$. We define 
\begin{equation}
f(\vec x,c_1,c_2,t) := \left\{
\begin{array}{ll}
(u_0(\vec x)-c_1)^2, & \text{if } \,\vec x \in \Omega_1(t), \\
(u_0(\vec x)-c_2)^2, & \text{if } \,\vec x \in \Omega_2(t), 
\end{array}
\right.
\end{equation}
which is defined for a.e. $\vec x \in \Omega$. 

By using a transport theorem, we obtain 
\begin{align*}
&\left.\frac{\mathrm{d}}{\mathrm{d}t}\right|_{t=0} E(\Gamma(t),c_1,c_2) = \\
&= \left.\frac{\mathrm{d}}{\mathrm{d}t}\right|_{t=0} \left(\sigma \int_{\Gamma(t)} 1 \,\mathrm{d}A + \lambda \int_{\Omega} f(\vec x,c_1,c_2,t) \,\mathrm{d}x\right) \\
&= - \sigma \int_\Gamma \kappa\,V_n\,\mathrm{d}A + \\
& \quad - \lambda \int_\Gamma \left((u_0-c_1)^2 -(u_0-c_2)^2 \right) \,V_n\,\mathrm{d}A \\
&= - \int_\Gamma (\sigma \kappa + F)\,V_n \,\mathrm{d}A
\end{align*}
where $\mathrm{d}A$ is the area element, $V_n$ is the normal velocity, $\kappa$ the mean curvature and $F$ is an external force given by 
\begin{equation}
F(\vec x) = \lambda \left((u_0(\vec x)-c_1)^2 -(u_0(\vec x)-c_2)^2 \right), \, \vec x \in \Gamma.
\label{eq:definition_F_twophase_3D}
\end{equation}
The fastest decrease of the energy is obtained for 
\begin{equation}
V_n = \sigma \kappa + F.
\label{eq:evolution_eq_twophase_3D}
\end{equation} 

Also multichannel images with a vector-valued image function $\vec u_0: \Omega \rightarrow \mathbb{R}^d$ can be handled. This involves vector-valued coefficients $\vec c_k$, $k=1,2$, and a modification of the external force to, for example,
\begin{equation}
F(\vec x) = \sum_{i=1}^d \lambda_i \left[((u_0)_i(\vec x)-(c_1)_i)^2 -((u_0)_i(\vec x)-(c_2)_i)^2 \right],
\end{equation}
where the subscript $i$ denotes the $i$-th component of a vector, $i=1,\ldots,d$. For computation of the coefficients, each component of $\vec c_k$ is set to the mean of the corresponding component of $\vec u_0$ in the region $\Omega_k$, $k=1,2$. 

In principle, also spaces like the HSV (hue, saturation, value) or CB (chromaticity, brightness) space can be used \cite{Aujol06}, \cite{Chan01_2}, \cite{Tang02}. In these cases, the image function has values on certain submanifolds of $\mathbb{R}^d$. In \cite{Benninghoff2014a}, we proposed a method to segment 2D images using the color space HSV and CB. The method can be transferred also to the 3D case. In many practical applications however, for example medical 3D image data generated by computed tomography (CT) or magnetic resonance imaging (MRT), the image function is often scalar-valued (cf. for example the lung image database of The Cancer Imaging Archive (TCIA) \cite{Reeves2007}, \cite{Armato2011}, \cite{Reeves2011}).

\subsection{Parametric and Multiphase Formulation}
\label{subsec:3d_parametric_multiphase}
Equation \eqref{eq:evolution_eq_twophase_3D} can be rewritten using a parametric approach to describe the time-dependent surfaces. Further, we now consider a more general setup of multiple surfaces $\Gamma_i(t)$, $t \in [0,T]$, $i=1, \ldots, N_S$, which separate three-dimensional regions $\Omega_k(t)$, $k=1, \ldots, N_R$. We assume that the surfaces are compact and oriented by unit normal vector fields $\vec \nu_i(\,.\,,t)$ pointing from $\Omega_{k^-(i)}(t)$ to $\Omega_{k^+(i)}(t)$, where $k^\pm(i)\in \{1, \ldots, N_R\}$. 

Let $\vec x_i(\,.\,,t): \Upsilon_i \rightarrow \mathbb{R}^3$, $i=1,\ldots,N_S$, be a smooth parameterization of $\Gamma_i(t)$, where $\Upsilon_i$ is a two-dimensional reference manifold, for example the sphere $\Upsilon_i = S^2 \subset \mathbb{R}^3$. The normal velocity of $\Gamma_i(t)$ can be expressed as $(V_n)_i = (\vec x_i)_t \,.\,\vec \nu_i$. 

An approximation of the image intensity function $u_0$ is given by the piecewise constant function $u(\,.\,,t) = \sum_{k=1}^{N_R} c_k(t) \chi_{\Omega_k(t)}$, where $\chi_{\Omega_k(t)}$ is the characteristic function  of $\Omega_k(t)$ and $c_k(t)$ is the mean of $u_0$ in $\Omega_k(t)$. 

For each surface, we define the external forcing term
\begin{equation}
F_i(\,.\,,t) = \lambda \left((u_0-c_{k^+(i)}(t))^2 -(u_0-c_{k^-(i)}(t))^2 \right)
\label{eq:def_F_i_3d}
\end{equation}
and obtain the following scheme for the surfaces: Find $\vec x_i(\,.\,,t): \Upsilon_i \rightarrow \mathbb{R}^3$ and $\kappa_i(\,.\,,t): \Upsilon_i \rightarrow \mathbb{R}$, $i=1, \ldots, N_S$, satisfying
\begin{subequations}
\label{eq:strong_scheme_3D}
\begin{align}
(\vec x_i)_t \,.\, \vec \nu_i &= \sigma \kappa_i + F_i, \label{eq:scheme_strong_3d_1}\\
\Delta_\Gamma \vec x_i &= \kappa_i \vec\nu_i.\label{eq:scheme_strong_3d_2}
\end{align}
\end{subequations}
Equation \eqref{eq:scheme_strong_3d_1} is a parametric formulation of \eqref{eq:evolution_eq_twophase_3D} for multiple regions. Equation \eqref{eq:scheme_strong_3d_2} relates the parametrization $\vec x_i$ and the curvature $\kappa_i$, see e.g. \cite{DeckelnickDziukElliott05}. The symbol $\Delta_\Gamma$ denotes the Laplace-Beltrami operator. Here, we use a small abuse of notation, i.e. we consider $\kappa_i$ and $\vec\nu_i$ as functions defined on $\Upsilon_i$, i.e. we identify $\kappa_i$ with $\kappa_i \circ \vec x_i$ and $\vec\nu_i$ with $\vec\nu_i \circ \vec x_i$, $i=1,\ldots,N_S$.  

%\subsection{Weak Scheme}
%Introduction of Weak Scheme, basis for finite element formulation

%Text with citations \cite{Mumford89}.
%\subsection{Subsection title}
%
%as required. Don't forget to give each section
%and subsection a unique label (see Sect.~\ref{sec:segmentation}).
%\paragraph{Paragraph headings} Use paragraph headings as needed.
%\begin{equation}
%\end{equation}
%
%% For one-column wide figures use
%\begin{figure}
%% Use the relevant command to insert your figure file.
%% For example, with the graphicx package use
  %\includegraphics[width = 0.2\textwidth]{fig/siegel.pdf}
%\end{figure}
%%
%% For two-column wide figures use
%\begin{figure*}
%% Use the relevant command to insert your figure file.
%% For example, with the graphicx package use
  %\includegraphics[width=0.4\textwidth]{fig/siegel.pdf}
%\end{figure*}
%\begin{table}
%\begin{tabular}{lll}
%\hline\noalign{\smallskip}
%\end{tabular}
%\end{table}
%

%% file: 3_numerics.tex
\section{Numerical approximation}
\label{sec:numerics}

\subsection{Finite Element Approximation}
\label{sec:threedim_fe_appr}
We introduce a finite element approximation for the scheme \eqref{eq:strong_scheme_3D} which is based on a scheme developed in \cite{BGN08c}, where geometric flows of two-dimensional surfaces are considered. We extend the ideas to solve schemes like  \eqref{eq:strong_scheme_3D} which arise in image segmentation applications. 

Let $0=t_0 < t_1 < \ldots < t_M = T$ be a decomposition of the time interval into possibly variable time steps $\tau_m = t_{m+1}-t_m$ for $m=0, \ldots, M-1$.  

Let $N_S$ denote the number of surfaces and $N_R$ denote the number of regions. Let the smooth surface $\Gamma_i(t_m)$, $i=1, \ldots, N_S$, be approximated by a polyhedral surface $\Gamma_i^m$ of the form
\begin{equation}
\Gamma_i^m = \bigcup_{j=1}^{N_{i,F}} \overline{\sigma_{i,j}^m},
\end{equation}
where $\sigma_{i,j}^m$, $j=1, \ldots, N_{i,F}$, are disjoint, open simplices (also called faces) with vertices $\vec q_{i,j}^m$, $j=1, \ldots, N_{i,V}$. Further, let $h := \mathrm{max}_{i=1,\ldots,N_S, j=1, \ldots, N_{i,F}} \mathrm{diam}(\sigma_{i,j}^m)$ be the maximum diameter of a simplex of the triangulated surfaces. The diameter $\mathrm{diam}(\sigma_{i,j}^m)$ is defined as the maximum distance between two points of $\overline{\sigma_{i,j}^m}$. 

Let $\vec X_i^m$ be a parameterization of $\Gamma_i^m$ and let $\Omega_k^m$, $k=1,\ldots,N_R$, denote the open, disjoint subsets of $\Omega$ separated by $\Gamma_i^m$, $i=1, \ldots, N_S$. Thus, $\Omega_k^m$ is an approximation of $\Omega_k(t_m)$ for $k=1,\ldots,N_R$. 

The new surfaces $\Gamma_i^{m+1}$ are parameterized over $\Gamma_i^m$.  Therefore, we define the following finite element spaces
\begin{subequations}
\begin{align}
W(\Gamma^m) :=& \left\{(\eta_1, \ldots, \eta_{N_S}) \in C(\Gamma_1^m,\mathbb{R})\times \ldots \times \right.\nonumber\\
&\quad  C(\Gamma_{N_S}^m,\mathbb{R})\,: \eta_i|_{\sigma_{i,j}^m}\, \text{ is linear, }\, \nonumber\\
& \quad \left. \forall i=1,\ldots, N_S, \, j=1,\ldots,N_{i,F}\right\}, \\
\underline{V}(\Gamma^m) :=& \left\{(\vec\eta_1, \ldots, \vec\eta_{N_S}) \in C(\Gamma_1^m,\mathbb{R}^3)\times \ldots \times \right. \nonumber \\
& \quad  C(\Gamma_{N_S}^m,\mathbb{R}^3) \,: \vec\eta_i|_{\sigma_{i,j}^m}\, \text{ is linear, }\, \nonumber\\
& \quad \left. \forall i=1,\ldots, N_S, \, j=1,\ldots,N_{i,F}\right\}.
\end{align}
\end{subequations}
The spaces $W(\Gamma^m)$ and $\underline{V}(\Gamma^m)$ thus consist of scalar or vector-valued, piecewise linear functions defined on $\Gamma^m$. 

A basis of $W(\Gamma^m)$ is given by functions $\chi_{i,j}^m := ((\chi_{i,j}^m)_1, \ldots, (\chi_{i,j}^m)_{N_S}) \in W(\Gamma^m)$, where 
\begin{equation}
(\chi_{i,j}^m)_k(\vec q_{k,l}^m)= \delta_{ik}\delta_{jl}
\end{equation}
for $i,k=1, \ldots, N_S$, $j=1, \ldots, N_{i,V}$, $l=1, \ldots, N_{k,V}$. 

Note, that depending whether the domain of definition is $\Gamma^{m-1}$ or $\Gamma^m$, we can interpret $\vec X^m$ as a different function. In particular we have  for $m\geq 1$, $\vec X^m \in \underline{V}(\Gamma^{m-1})$, and for $m\geq 0$, $\vec X^m \in \underline{V}(\Gamma^m)$ is the identity defined on $\Gamma^m$. 

For scalar functions $u=(u_1, \ldots, u_{N_S})$, $v = (v_1, \ldots, $ \,$v_{N_S}) \in L^2(\Gamma_1^m, \mathbb{R}) \times \ldots \times L^2(\Gamma_{N_S}^m, \mathbb{R})$ and for vector-valued functions $u=(u_1, \ldots, u_{N_S})$, $v = (v_1, \ldots, v_{N_S})\in$ \,\, $L^2(\Gamma_1^m, \mathbb{R}^{3}) \times \ldots \times L^2(\Gamma_{N_S}^m, \mathbb{R}^{3})$, we introduce the $L^2$-inner product over the current polyhedral surface $\Gamma^m$ as follows:
\begin{equation}
\langle u,v \rangle_m := \int_{\Gamma^m} u\,.\, v \,\mathrm{d}A = \sum_{i=1}^{N_S} \int_{\Gamma_i^m} u_i \,.\, v_i \,\mathrm{d}A.
\end{equation}

If $u,v$ are piecewise continuous with possible jumps across the edges of $\sigma_{i,j}^m$, $i=1, \ldots, N_S$, $j=1, \ldots, N_{i,F}$, the mass lumped inner product is defined as
\begin{equation}
\langle u,v \rangle_m^h := \frac13 \sum_{i=1}^{N_S}\sum_{j=1}^{N_{i,F}} |\sigma_{i,j}^m| \sum_{l=1}^{3} (u\,.\,v)((\vec q_{i,j_l}^m)^-),
\end{equation}
where $\vec q_{i,j_l}^m$, $l=1,2,3$, are the vertices of $\sigma_{i,j}^m$, $|\sigma_{i,j}^m| = \frac12 \| (\vec q_{i,j_2}^m - \vec q_{i,j_1}^m) \times (\vec q_{i,j_3}^m - \vec q_{i,j_1}^m)\|$ is the area of $\sigma_{i,j}^m$ and $u((\vec q_{i,j_l}^m)^-) := \mathrm{lim}_{\vec p \rightarrow \vec q_{i,j_l}^m, \,\vec p \in \sigma_{i,j}^m} u(\vec p)$.

We assume that the vertices $\left\{\vec q_{i,j_l}^m\right\}_{l=1}^3$, $j=1, \ldots,$ $N_{i,F}$,    are ordered 
such that the unit normal $\vec \nu_i^m$ at $\Gamma_i^m$ is given by 
\begin{equation}
\vec \nu_i^m |_{\sigma_{i,j}^m} := \vec \nu_{i,j}^m :=\frac{(\vec q_{i,j_2}^m - \vec q_{i,j_1}^m) \times (\vec q_{i,j_3}^m - \vec q_{i,j_1}^m)}{\|(\vec q_{i,j_2}^m - \vec q_{i,j_1}^m) \times (\vec q_{i,j_3}^m - \vec q_{i,j_1}^m)\|}
\end{equation}
points from $\Omega_{k^-(i)}^m$ to $\Omega_{k^+(i)}^m$.

We propose the following finite element scheme approximating the scheme \eqref{eq:strong_scheme_3D}: Let $\Gamma^0$ be a union of polyhedral surfaces approximating $\Gamma(0)$ and let $\vec X^0 \in \underline{V}(\Gamma^0)$ be the identity function on $\Gamma^0$. Find $\vec X^{m+1} \in \underline{V}(\Gamma^m)$ and $\kappa^{m+1} \in W(\Gamma^m)$, $m=0,1,\ldots,M-1$, such that
\begin{subequations}
\label{eq:fem_scheme_3d}
\begin{equation}
\langle \frac{\vec X^{m+1}-\vec X^m}{\tau_m}, \chi \,\vec\nu^m \rangle_m^h - \sigma\langle \kappa^{m+1}, \chi \rangle_m^h = \langle F^m, \chi \rangle_m^h, \nonumber 
\end{equation}
\begin{equation}
\forall \chi \in W(\Gamma^m), \label{eq:fem_scheme_3d_1}
\end{equation}

\begin{equation}
\langle \kappa^{m+1} \,\vec\nu^m, \vec \eta \rangle_m^h + \langle \nabla_s \vec X^{m+1}, \nabla_s \vec\eta\rangle_m = 0, \nonumber
\end{equation}
\begin{equation}
\forall \vec\eta \in \underline{V}(\Gamma^m).\label{eq:fem_scheme_3d_2}
\end{equation}
\end{subequations}

\vspace{1ex}
Here,  $F^m = (F_1^m, \ldots, F_{N_S}^m)$ is defined by  
\begin{align*}
F_i^m(\vec q_{i,j}^m):=&\lambda \left((u_0(\vec X_i^m(\vec q_{i,j}^m)) - c_{k^+(i)}^m)^2 \right. +\\
& \left. - (u_0(\vec X_i^m(\vec q_{i,j}^m)) - c_{k^-(i)}^m)^2\right), 
\end{align*}
for $i=1,\ldots, N_S$ and $j=1, \ldots, N_{i,V}$ and $c_k^m$ is set to the mean of $u_0$ in $\Omega_k^m$ for $k=1,\ldots,N_R$. 

We further introduce a weighted normal defined at the nodes $\vec X_i^m(\vec q_{i,j}^m) = \vec q_{i,j}^m \in \Gamma_i^m$ by setting
\begin{equation}
\vec \omega_i^m(\vec q_{i,j}^m) := \vec\omega_{i,j}^m := \frac{1}{|\Lambda_{i,j}^m|} \sum_{\sigma_{i,l}^m \in \mathcal{T}_{i,j}^m} |\sigma_{i,l}^m| \,\vec\nu_{i,l}^m, 
\label{eq:omega_ij_m_3d}
\end{equation}
where for $i=1,\ldots,N_S$ and $j=1, \ldots, N_{i,V}$, $\mathcal{T}_{i,j}^m :=\left\{ \sigma_{i,l}^m \,:\, \vec q_{i,j}^m \in \overline{\sigma_{i,l}^m}\right\}$ and $\Lambda_{i,j}^m := \bigcup_{\sigma_{i,l}^m \in \mathcal{T}_{i,j}^m} \overline{\sigma_{i,l}^m}$. 

Further, we set $\vec v_i^m(\vec q_{i,j}^m):= \vec v_{i,j}^m :=\vec \omega_{i,j}^m / \| \vec \omega_{i,j}^m \|$ and $\vec\omega^m = (\vec \omega_1^m, \ldots, \vec\omega_{N_S}^m)$ and $\vec v^m = (\vec v_1^m, \ldots, \vec v_{N_S}^m)$. 

As in \cite{BGN08b}, \cite{BGN08c}, we make a very mild assumption on the triangulations:
\begin{itemize}
\item[$(\mathcal{A})$] For $m=0, \ldots, M$, we assume that $|\sigma_{i,j}^m| > 0$ for all $i=1,\ldots,N_S$ and $j=1, \ldots, N_{i,F}$ and for $m=0, \ldots, M-1$, we assume that $\mathrm{dim}\,\mathrm{span} \left\{\vec\omega_{i,j}^m\right\}_{j=1}^{N_{i,V}}=3$. 
\end{itemize}
The assumption $(\mathcal{A})$ is only violated in very rare cases. For closed surfaces without self intersections, it always holds. 

\begin{theorem}
Let the assumption $(\mathcal{A})$ hold. Then there exists a unique solution $\left\{\vec X^{m+1}, \kappa^{m+1}\right\}$ $\in \underline{V}(\Gamma^m) \times W(\Gamma^m)$ to the system \eqref{eq:fem_scheme_3d}. 
\end{theorem}
\begin{proof}
(See also \cite{BGN08c}.) Since the system is linear, it is sufficient to show uniqueness. Therefore, we consider the following scheme: Find $\vec X \in \underline{V}(\Gamma^m)$ and $\kappa \in W(\Gamma^m)$ such that 
\begin{subequations}
\begin{align}
-\frac{1}{\tau_m} \langle \vec X, \chi \,\vec\nu^m \rangle_m^h + \sigma\langle \kappa, \chi \rangle_m^h  &= 0, && \forall \chi \in W(\Gamma^m), \label{eq:fem_uniqueness_3d_1}\\
\langle \kappa \,\vec\nu^m, \vec \eta \rangle_m^h + \langle \nabla_s \vec X, \nabla_s \vec\eta\rangle_m &= 0, && \forall \vec\eta \in \underline{V}(\Gamma^m),\label{eq:fem_uniqueness_3d_2}
\end{align}
\end{subequations}
holds. Testing \eqref{eq:fem_uniqueness_3d_1} with $\chi = \kappa$ and \eqref{eq:fem_uniqueness_3d_2} with $\vec\eta=\vec X$ leads to 
\begin{equation}
 \sigma \tau_m \langle \kappa, \kappa \rangle_m^h + \langle \nabla_s \vec X, \nabla_s \vec X\rangle_m = 0.
\end{equation}
It follows that $\kappa_{i,j} = 0$ and $\vec X_{i,j}=\vec C_i \in \mathbb{R}^3$ for $i=1, \ldots, N_S$, $j=1, \ldots, N_{i,V}$. Inserting $\kappa = 0$ and  $\vec X = \vec C=(\vec C_1, \ldots, \vec C_{N_S})$ in \eqref{eq:fem_uniqueness_3d_1} results in 
\begin{equation}
\langle \vec C, \chi \,\vec\nu^m \rangle_m^h = 0, \quad \forall \chi \in W(\Gamma^m).
\end{equation}
Choosing $\chi = \chi_{i,j}^m$ (the standard basis) and using \eqref{eq:omega_ij_m_3d}, the definition of $\vec\omega_{i,j}^m$, results in 
\begin{equation}
\vec C_i \,.\, \vec\omega_{i,j}^m = 0, \quad \forall i=1, \ldots, N_S, \,j=1,\ldots, N_{i,V}.
\end{equation}
Finally, using the assumption $(\mathcal{A})$, it follows that $\vec C_i=0$ for each $i=1, \ldots, N_S$.
\qed
\end{proof}

\subsection{Solution of the Discrete System}
\label{subsec:solution_discrete_system_3d}

We define $\delta \vec X^{m+1} := \vec X^{m+1}-\vec X^m$. As $\delta \vec X^{m+1}$ and $\kappa^{m+1}$ are uniquely given by their values at the nodes $\vec q_{i,j}^m$, we consider them as elements in $(\mathbb{R}^3)^N$ and $\mathbb{R}^N$, respectively, where $N=\sum_{i=1}^{N_S}N_{i,V}$. We introduce the matrices $M_m \in \mathbb{R}^{N\times N}$, $\vec N_m \in (\mathbb{R}^3)^{N\times N}$ and $\vec A_m \in (\mathbb{R}^{3 \times 3})^{N\times N}$ by
\begin{equation*}
M_m:= \left(
\begin{array}{ccc}
M_m^1 & \cdots & 0 \\
\vdots &  \ddots & \vdots \\
0 & \ldots & M_m^{N_S} 
\end{array}
\right),
\end{equation*}
\begin{equation*}
\vec N_m:= \left(
\begin{array}{ccc}
\vec N_m^1  & \cdots & 0 \\
\vdots  & \ddots & \vdots \\
0 &  \ldots & \vec N_m^{N_S} 
\end{array}
\right),
\end{equation*}
\begin{equation*}
\vec A_m:= \left(
\begin{array}{ccc}
\vec A_m^1  & \cdots & 0 \\
\vdots &  \ddots & \vdots \\
0 &  \ldots & \vec A_m^{N_S} 
\end{array}
\right),
\end{equation*}
where $M_m^i \in \mathbb{R}^{N_{i,V} \times N_{i,V}}$, $\vec N_m^i \in (\mathbb{R}^3)^{N_{i,V} \times N_{i,V}}$, $\vec A_m^i \in (\mathbb{R}^{3\times 3})^{N_{i,V} \times N_{i,V}}$, $i=1, \ldots, N_S$. Their entries are defined by
\begin{align}
[M_m^i]_{kl} &:= \langle \chi_{i,k}^m, \chi_{i,l}^m \rangle_m^h,\nonumber \\
[\vec N_m^i]_{kl} &:= \langle \chi_{i,k}^m, \chi_{i,l}^m \vec\nu^m \rangle_m^h, \label{eq:def_matrices_MNA_3d} \\
[\vec A_m^i]_{kl} &:= \langle \nabla_s \chi_{i,k}^m, \nabla_s \chi_{i,l}^m \rangle_m \,\vec{\mathrm{Id}}_{3},\nonumber
\end{align}
with $i=1,\ldots,N_S$, $k,l=1,\ldots, N_{i,V}$. Here, $\vec{\mathrm{Id}}_{3}$ denotes the identity matrix in $\mathbb{R}^{3\times 3}$. Further, we introduce $b_m = (b_m^1, \ldots, b_m^{N_S})\in \mathbb{R}^N$ defined by
\begin{equation}
[b_m^i]_k := \langle F_i^m, \chi_{i,k}^m \rangle_m^h, \quad i=1,\ldots,N_S,\,k=1,\ldots,N_{i,V}.
\label{eq:def_b_m_3d}
\end{equation}
The scheme \eqref{eq:fem_scheme_3d} can be rewritten to the following problem: Let $\Gamma^0$ be a polyhedral approximation of $\Gamma(0)$ and let $\vec X^0=(\vec X_1^0, \ldots, \vec X_{N_S}^0) \in (\mathbb{R}^3)^N$ with $\vec X_i^0 = (\vec X_{i,1}^0, \ldots, \vec X_{i,N_{i,V}}^0)$ such that $\vec X_{i,j}^0$ are the coordinates of the vertices of $\Gamma_i^0$ for $i=1,\ldots,N_S$, $j=1,\ldots,N_{i,V}$. 
For $m=0, \ldots, M-1$ find $\delta \vec X^{m+1}\in(\mathbb{R}^3)^N$ and $\kappa^{m+1}\in\mathbb{R}^N$ such that 
\begin{equation}
\left(
\begin{array}{cc}
\sigma \tau_m M_m & -\vec N_m^T \\
\vec N_m   & \vec A_m
\end{array}
\right)
\left(
\begin{array}{c}
\kappa^{m+1} \\ \delta \vec X^{m+1} 
\end{array}
\right)
= 
\left(
\begin{array}{c}
-\tau_m b_m \\ -\vec A_m \vec X^m 
\end{array}
\right).
\end{equation}
Applying a Schur complement approach, we can transform this system to 
\begin{subequations}
\label{eq:schur_3d}
\begin{equation}
\kappa^{m+1} = \frac{1}{\sigma} M_m^{-1} \left(\frac{1}{\tau_m} \vec N_m^T \delta \vec X^{m+1} - b_m \right), 
\end{equation}
\begin{align}
&\left(\frac{1}{\sigma\tau_m} \vec N_m M_m^{-1} \vec N_m^T + \vec A_m\right) \delta \vec X^{m+1} = \nonumber \\
&=  - \vec A_m \vec X^m + \frac{1}{\sigma}\vec N_m M_m^{-1}b_m. \label{eq:schur2_3d}
\end{align}
\end{subequations}
Since the system matrix in \eqref{eq:schur2_3d} is symmetric and positive definite under the assumption $(\mathcal{A})$, there exists a unique solution. 

The linear system \eqref{eq:schur2_3d} can be solved with an iterative solver, for example, with the method of conjugate gradients with possible preconditioning, or with a direct solver for sparse matrices. For the experiments and examples presented in  Section~\ref{sec:results} of this paper, we use a MATLAB built-in routine, a direct solver for sparse systems. Even for two-dimensional problems, which result from the evolution of two-dimensional surfaces, the sparse direct solver is very efficient from a computational view.

\subsection{Topology Changes}
\label{subsec:top_changes_3d}
Parametric methods cannot handle topology changes automatically in contrast to other numerical methods like the level set method. During the evolution of surfaces singularities can occur like a pinch-off, see \cite{Baensch05}, \cite{BGN08c}. In order to proceed after a pinch-off, the surface has to be split in two single surfaces. Other possible topology changes, that we will consider here, are merging of two surfaces and change of the genus (occurs for example during an evolution of a torus to a sphere or vice versa).

In \cite{MikulaUrban12}, an algorithm is proposed to efficiently detect splitting and merging of evolving curves in $\mathbb{R}^2$, see also \cite{Balazovjech12}. In \cite{Benninghoff2014a}, we used and extended this algorithm to detect topology changes in 2D images. In this paper, we want to adapt this approach to the 3D case.  We aim at detecting topology changes which could occur during the evolution of surfaces. Having found the location where a topology change occurs, we propose a method how to modify the triangulations.

\subsubsection{Detection of a Topology Change}
\label{subsubsec:detection_topology_change_3D}
In \cite{MikulaUrban12}, a virtual, auxiliary 2D background grid is constructed which covers a two-dimensional domain, and topology changes of curves are detected if node points from different curves or different parts of one curve are located in one array of the background grid. 

Motivated by this method for evolving curves, we propose the following method to detect topology changes of evolving surfaces. The basic idea is the use of a uniform 3D grid of cubes. A topology change may occur, if a large number of nodes or if nodes of different surfaces or different parts of one surface (with opposite normal vector) are located in one cube. 

In detail, to detect a change in topology, we construct a uniform 3D background grid which covers the image domain $\Omega$. In the following, we assume that $\Omega$ is a cuboid. If the 3D image $u_0$ is not given on a cuboid volume, we consider a cuboid which contains $\Omega$. 
 
Let $\Omega = [x_\mathrm{min},x_\mathrm{max}]\times [y_\mathrm{min},y_\mathrm{max}] \times [z_\mathrm{min},z_\mathrm{max}] \subset \mathbb{R}^3$ be the image domain containing in particular the surfaces $\Gamma_i^m$, $i=1, \ldots, N_S$. We consider a  grid dividing $\Omega$ in a set of many small cubes of edge width $a \in \mathbb{R}$. Let the grid consist of $N_x \times N_y \times N_z$ cubes, where $N_x = \mathrm{ceil}((x_\mathrm{max}-x_\mathrm{min})/a)$, $N_y = \mathrm{ceil}((y_\mathrm{max}-y_\mathrm{min})/a)$ and $N_z = \mathrm{ceil}((z_\mathrm{max}-z_\mathrm{min})/a)$. 

We now perform one loop over all surfaces and nodes $\vec X_{i,j}^m$, $i=1,\ldots,N_S$, $j=1,\ldots, N_{i,V}$. If a node $\vec X_{i,j}^m$ is the first node which is detected to lie in a certain cube, we create a new list for that cube, where we store the index pair $(i,j)$. If  another node has already been identified to be located inside that cube, we add the index pair $(i,j)$ to the existing list. 

If there is a large number of nodes  located in a cube, i.e. more than $N_\mathrm{detect}$ nodes, a topology change likely occurs, and the cube is stored in a list for possible topology changes. It is also possible that the node density is only locally very high at this location, but no topology change happens. 

If there are less than $N_\mathrm{detect}$ nodes in the current list, we compare the surface index and the direction of the weighted normal vector of the current node with those of the nodes already stored in the list. If two different surface indices $i_1$ and $i_2$ occur or if two nodes with (nearly) opposite weighted normal vector are located in one cube, a topology change likely happens. The corresponding cube is accordingly stored in a list for topology changes. 

After having considered all nodes, the cubes marked for topology changes are considered one by one. If a topology change is identified, the surface triangulation is accordingly changed. Thus, by successively considering all marked cubes, more than one topology change can be executed in one time step. The list can be optionally sorted such that the cube with the largest number of nodes is considered first.  

The detection of topology changes is very efficient from a computational view. The effort is of order  $\mathcal{O}(N)$, where $N$ is the total number of node points. For comparison, a simple approach, where  all possible pairs of two nodes are considered and where the distance between two nodes is computed to detect a topology change, would result in a computational effort of order $\mathcal{O}(N^2)$. 

Similar as described in \cite{Benninghoff2014a}, the grid size $a$ can be adaptively set, for example dependent on the speed of the evolving surfaces. Therefore the method to detect topology changes is both efficient and robust. 

\subsubsection{Identification of the Topology Change}
\label{subsec:identification_top_change_3D}
For identifying which kind of topology change occurs, the nodes of the affected cube, i.e. the current cube of the sorted list, and the nodes of up to 26 neighbor cubes  (in total up to 27 cubes, i.e. $3 \times 3 \times 3$ cubes) are considered. Let $S = \{j_1, \ldots, j_{n_c}\}$ denote the index set of the nodes and let $\vec X_j$, $j \in S$, denote the coordinates of the nodes located in the   cubes. Further, let $\vec\omega_j$, $j\in S$, denote the corresponding weighted normal vectors at  $\vec X_j$, recall their definition in \eqref{eq:omega_ij_m_3d}. For the ease of illustration, we omit the time dependency (time index $m$) in the notation. 
 
The different topology changes are distinguished by considering the weighted normal vectors $\vec \omega_j$. The idea is that in case of merging and in case of increasing genus, their are two main group of nodes which can be found by considering their normal vectors. For splitting and decrease of genus, there are more than two main directions. 

The node with index $j=j_1$ is set as representative of the first group. We choose thresholds $thr1 < thr2$, for example $thr1 = 20^\circ$, $thr2 = 160^\circ$. For $j=j_2,\ldots, j_{n_c}$ we consider the angle $\alpha$ between $\vec \omega_{j_1}$ and $\vec\omega_j$. If $\alpha < thr1$, the node $j$ belongs to the first group. If $\alpha > thr2$ and if the second group is empty, the node $j$ becomes the representative node of the second group. If the second group is not empty, we consider the angle $\beta$ between $\vec \omega_j$ and $\vec\omega_{k_0}$, where $k_0 \in \{j_2, \ldots, j_{n_c}\}$ is the representative of the second group. If $\beta < thr1$, $j$ is added to the second group. 

For the next search we replace $\vec\omega_{j_1}$ and $\vec\omega_{k_0}$ by the average normal vectors $\vec n_1$ and $\vec n_2$ of group 1 and group 2, respectively, and re-consider the nodes which could not have been assigned to one group in the first step. If the angle between $\vec\omega_j$ and $\vec n_1$  or $\vec\omega_j$ and $\vec n_2$ is smaller than $thr1$, the node $\vec X_j$ is added to the corresponding group. 

If group 2 is empty, or if one of the groups consists of only a small number of nodes (e.g. $< 5\%$ of $n_c$), we start again by using another node as representative for the first group (e.g. $j=j_2$), since the node $j=j_1$ could be an outlier. If no start node can be found, such that there exist two groups of nodes as described above, no topology change takes place. This can happen, if all weighted normals point in nearly the same direction. 

The method such provides that a topology change like splitting is not wrongly detected at locations where several nodes are just close to each other but their normal vectors point in the same direction. From a mesh quality point of view, such meshes should be avoided; the nodes should be distributed equally over the surface. However, the detection of topology changes should be robust enough not to wrongly detect a splitting at locations where there is just a high density of nodes. 

If both groups have a sufficient number of nodes, we proceed by considering the remaining normal vectors which could have not been assigned to one of the two groups. We again consider the angle between $\vec\omega_j$ and $\vec n_1$ and $\vec\omega_j$ and $\vec n_2$. If both angles are $> thr3$ (e.g. $thr3 = 40^\circ$), the normal $\vec\omega_j$ points in a complete different direction compared to $\vec n_1$ and $\vec n_2$. Let $N_0$ be the number of such points. 
If $N_0$ exceed a predefined number (e.g. $1/3 \,n_c$), then there are more than two main groups of directions. In this case, there is a splitting or decrease of genus. Splitting and decrease of genus are handled similarly (see below for details how to modify the triangulations). Triangles close to the detected cube are deleted. If the remaining triangulation of the former surface consists of two connected components, a splitting occurs. If the remaining triangulation is connected, a decrease of genus occurs.

If $N_0$ is zero or if it does not exceed the predefined number, possible remaining normal vectors are only single outliers and there are only two main groups of normal vectors with nearly opposed normal vector. In this case, there is a merging or increase of genus. If there are nodes belonging to two different surfaces, a merging occurs. Otherwise, an increase of genus occurs. 

An illustration of the algorithm to detect and identify topology changes is given in Figure~\ref{fig:top_change_diagram_3d}. 

% Flow diagram top change algorithm

% Define block styles
%\tikzstyle{decision} = [diamond, draw, fill=green!20, 
    %text width=1.8cm, text centered, node distance=0.7cm]
\tikzstyle{decision} = [rectangle, draw, fill=gray!20, 
    text width=2cm, text centered, node distance=0.7cm, minimum height = 1.2em]
				
\tikzstyle{block1} = [rectangle, draw, 
    text width=5cm, text centered, minimum height=1.2em]
\tikzstyle{block2} = [rectangle, draw, 
    text width=3cm, text centered, minimum height=1.2em]
\tikzstyle{block3} = [rectangle, draw, 
    text width=1.5cm, text centered, minimum height=1.2em]
		
\tikzstyle{line} = [draw, -latex']

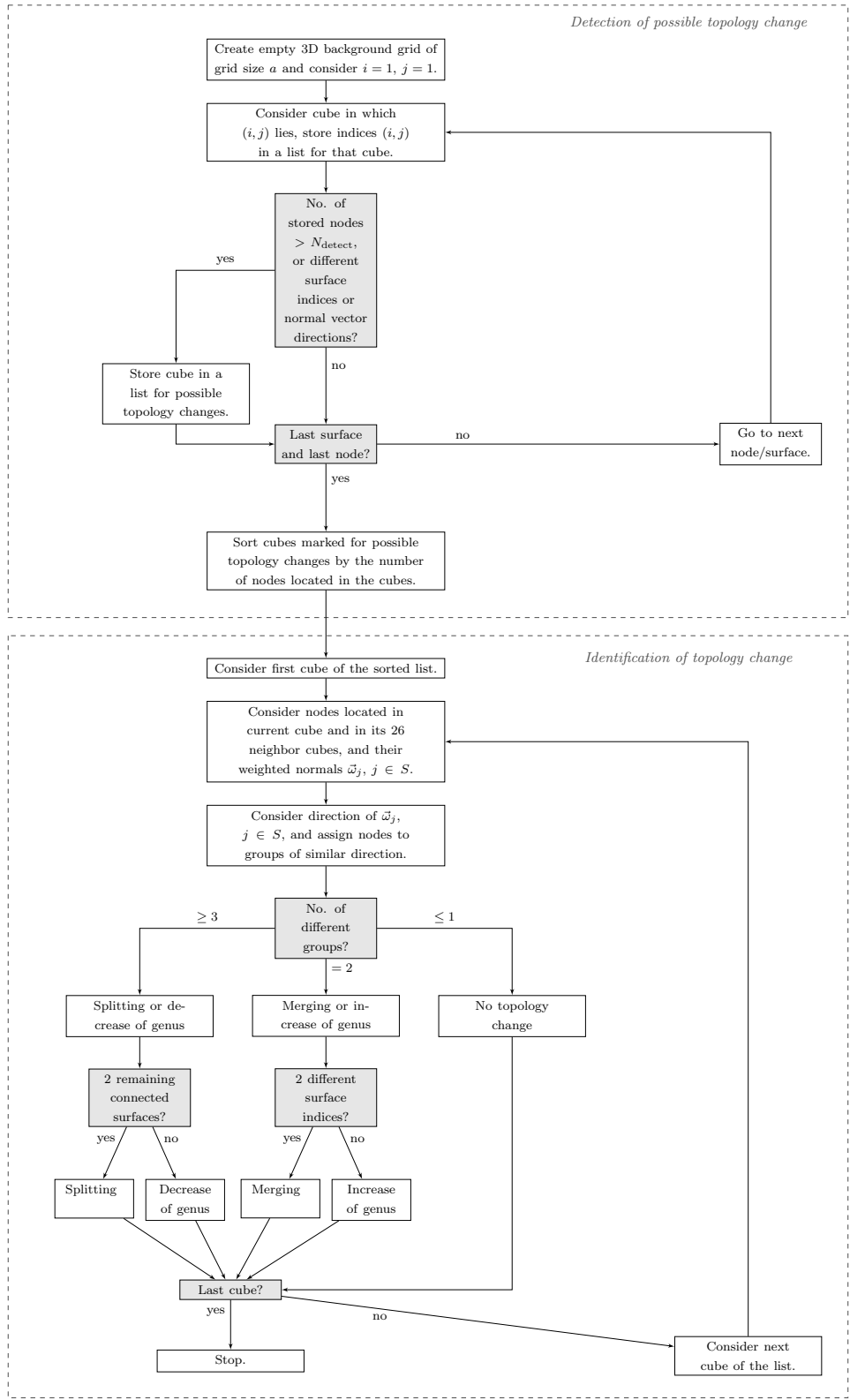
\begin{figure*}
\begin{center}
\resizebox{!}{21cm}{%
\begin{tikzpicture}[node distance = 0.5cm, auto]
\draw [dashed, color=black!70] (-7,1.2) rectangle (11.5,-12.3);
\draw [dashed, color=black!70] (-7,-12.7) rectangle (11.5,-29.5);

\draw[color=black!70] (8,0.8) node {\small{\textit{Detection of possible topology change}}};
\draw[color=black!70] (8,-13.2) node {\small{\textit{Identification of topology change}}};

%% STEP 1 - Detection
    % Place nodes
		\node [block1] (grid) {\footnotesize{Create empty 3D background grid of grid size $a$ and consider $i=1$, $j=1$.}};
		\node [block1, below =of grid] (xij) {\footnotesize{Consider cube in which $(i,j)$ lies, store indices $(i,j)$ in a list for that cube.}};
		\node [decision, below =of xij] (possible_top_change) {\footnotesize{No. of stored nodes $> N_\mathrm{detect}$, or different surface indices or normal vector directions?}}; 
		\node [block2, below left=of possible_top_change, xshift=-0.2cm] (store_cube) {\footnotesize{Store cube in a list for possible topology changes.}};
		\node [decision, below =of possible_top_change, yshift=-1cm] (last_node) {\footnotesize{Last surface and last node?}};
		\node [block1, below=of last_node, yshift = -1cm] (stop) {\footnotesize{Sort cubes marked for possible topology changes by the number of nodes located in the cubes.}};
		\node [block2, right=of last_node, xshift=7.05cm, text width = 2cm] (update) {\footnotesize{Go to next node/surface.}};

    % Draw edges
		\path [line] (grid) -- (xij);
		\path [line] (xij) -- (possible_top_change);
		\path [line] (possible_top_change) -| node [near start, above] {\footnotesize{yes}} (store_cube);
		\path [line] (possible_top_change) -- node [near start]  {\footnotesize{no}} (last_node);
		\path [line] (store_cube) |- (last_node);
	  \path [line] (last_node) -- node [near start] {\footnotesize{yes}} (stop);
		\path [line] (last_node) -- node [near start]  {\footnotesize{no}}  (update); 
		\path [line] (update) |- (xij);

%% STEP 2		
    % Place nodes
    \node [block1, below= of stop, yshift = -1cm] (first_cube) {\footnotesize{Consider first cube of the sorted list.}};
    \node [block1, below= of first_cube, ] (current_cube) {\footnotesize{Consider nodes located in current cube and in its 26 neighbor cubes, and their weighted normals $\vec\omega_j$, $j \in S$.}};
    \node [block1, below= of current_cube] (groups) {\footnotesize{Consider direction of $\vec\omega_j$, $j \in S$, and assign nodes to groups of similar direction.}};
    
		\node [decision, below =of groups] (group_decision) {\footnotesize{No. of different groups?}}; 
		\node [block2, below left=of group_decision, xshift = -1cm, yshift=-0.45cm] (splitting_decrease) {\footnotesize{Splitting or decrease of genus}};
		\node [block2, below=of group_decision, yshift=-0.3cm] (merging_increase) {\footnotesize{Merging or increase of genus}};
		\node [block2, below right=of group_decision, xshift = 1cm, yshift=-0.45cm] (no_top_change) {\footnotesize{No topology \\ change}};
    
		%\node [decision, below =of splitting_decrease] (decision1) {\footnotesize{Do $\vec\omega_j$ point away from each other?}}; 
		\node [decision, below =of splitting_decrease] (decision1) {\footnotesize{2 remaining connected surfaces?}}; 
		\node [block3, below left=of decision1,  xshift= 1.35cm, yshift = -0.8cm] (splitting) {\footnotesize{$ $ Splitting \newline $ $}};
		\node [block3, below right=of decision1, xshift=-1.35cm, yshift = -0.8cm] (decrease) {\footnotesize{Decrease of genus}};
    
		\node [decision, below =of merging_increase] (decision2) {\footnotesize{2 different surface indices?}}; 
		\node [block3, below left=of decision2,  xshift= 1.35cm, yshift = -0.8cm] (merging) {\footnotesize{$ $ Merging \newline $ $}};
		\node [block3, below right=of decision2, xshift=-1.35cm, yshift = -0.8cm] (increase) {\footnotesize{Increase of genus}};

    \node [decision, below =of decrease, xshift = 1cm, yshift = -0.6cm] (last_cube) {\footnotesize{Last cube?}};
		
    \node [block2, below =of last_cube, yshift = -0.6cm] (stop2) {\footnotesize{Stop.}};    
   \node [block2, below right=of last_cube, xshift = 8.3cm, yshift = -0.45cm] (next_cube) {\footnotesize{Consider next cube of the list.}};

   % Draw edges
   \path [line] (stop) -- (first_cube); 
   \path [line] (first_cube) -- (current_cube); 
   \path [line] (current_cube) -- (groups); 
   \path [line] (groups) --  (group_decision);
   
   \path [line] (group_decision) -| node [above, near start]  {\footnotesize{$\geq 3$}} (splitting_decrease);
   \path [line] (group_decision) -- node [near start]  {\footnotesize{$=2$}} (merging_increase);
   \path [line] (group_decision) -| node [above, near start]  {\footnotesize{$\leq 1$}} (no_top_change);

   \path [line] (splitting_decrease) --  (decision1);
   \path [line] (merging_increase) --  (decision2);
   \path [line] (decision1) -- node [left, near start]  {\footnotesize{yes}} (splitting);
   \path [line] (decision1) -- node [right, near start]  {\footnotesize{no}}  (decrease);
   \path [line] (decision2) -- node [left, near start]  {\footnotesize{yes}} (merging);
   \path [line] (decision2) -- node [right, near start]  {\footnotesize{no}}  (increase);

   \path [line] (no_top_change) |- (last_cube); 
	
   \path [line] (next_cube) |- (current_cube);

   \path [line] (splitting) --  (last_cube);
   \path [line] (decrease)  --  (last_cube);
   \path [line] (merging)   --  (last_cube);
   \path [line] (increase)  --  (last_cube);
	
   \path [line] (last_cube) -- node [left, near start] {\footnotesize{yes}} (stop2);
   \path [line] (last_cube)	-- node [below, near start] {\footnotesize{no}} (next_cube);

\end{tikzpicture}
}
\end{center}

\caption{Illustration of the detection and identification of topology changes of surfaces.}
\label{fig:top_change_diagram_3d}
\end{figure*}

\subsubsection{Algorithm for Splitting and Decreasing of Genus}
We propose the following algorithm for a possible modification of the surface triangulation after the detection of a splitting or decrease of genus.  

\begin{itemize}
\item \textbf{Preparation and deletion of simplices:} In case of splitting or decreasing genus, we consider the set of affected nodes $\vec X_j$, $j\in S$, which are located in the cube or in a neighbor cube, where the topology change has been detected. Let $\vec p_E$ be the mean of the points $\{\vec X_j \,:\, j\in S\}$. We delete all simplices with at least one vertex belonging to the set $\{\vec X_j\,:\, j \in S\}$.  When deleting one simplex, we change the neighbor information of neighbor simplices at the corresponding edges to $-1$ (free edges). Deletion creates two temporary holes in the surface(s). Simplices with two or three free edges are deleted as well, see Figure~\ref{fig:collision01}.  As a result we either have two sets of connected simplices ($\rightarrow$ splitting) or one set of connected simplices ($\rightarrow$ decrease of genus). 
\item \textbf{Set surface index (splitting only):} The remaining simplices form two connected sets. For one set, we need to re-set the surface index. Let $i$ be the surface index of the original surface. By splitting the total number of  surfaces is increased to $N_S+1$. Starting with one simplex with a free edge, we re-set its surface index from $i$ to $N_S+1$. Then, we consider its neighbor simplices and  assign them to surface $N_S+1$ also. By this procedure, the simplices of one of the two connected components are assigned one by one to the surface $N_S + 1$ by heritage, i.e. by use of neighbor information. 
\item \textbf{Generate new simplices:} We close each of the two intermediate holes (where simplices have been deleted) by constructing new simplices at edges of simplices where the neighbor information has been set to $-1$. First we create two new points, each with coordinates $\vec p_E$, one for each intermediate hole. In the next time steps the two nodes can move away from each other. If $\sigma$ is a simplex with a free edge given by $\vec X_{\sigma,j_1}$ and $\vec X_{\sigma,j_2}$ with no neighbor simplex, a new simplex is generated given by the vertices $\vec  X_{\sigma,j_1}$, $\vec X_{\sigma,j_2}$ and one of the two new nodes at $\vec p_E$. The new simplex inherits the surface index from $\sigma$. 
\item \textbf{Improve mesh quality:} By simply connecting all free edges with one of the two new vertices at $\vec p_E$, the two  vertices can belong to a big number of simplices. Let $\sigma_1$ and $\sigma_2$ be two of the newly generated simplices with one common edge. We construct $4$ new simplices from $\sigma_1$ and $\sigma_2$ such that the new vertex belongs to only one of the $4$ new simplices, see Figure~\ref{fig:4sigma}. Therefore the number of elements to which the vertices at $\vec p_E$ belong to is approximately halved.  We repeat this procedure until each of the newly created nodes at $\vec p_E$ belongs to $\leq 8$ elements. 
\end{itemize}

\begin{figure}[t]
\begin{center}
\begin{tikzpicture}[scale=0.35,transform shape]
\begin{scope}
\draw[line width=1.3pt, color=red] (-1,2) -- (0,1.2) -- (1,2) -- (3,0) -- (2, -1.5) -- (-1,-2) -- (-2,0) -- (-1,2);
\draw[line width=1.3pt, color=red] (0,0) -- (0.5, 0.8) -- (1,0) -- (0,0); 
\draw (-1,2) -- (0,3) -- (1,2) -- (3,1.8)-- (3,0) -- (3.3, -1.7) -- (2, -1.5) -- (1,-3)-- (-1,-2) -- (-2.5,-2) -- (-2,0) -- (-2.8,1.8) -- (-1,2);
\draw (-1,2) -- (1,2);

\draw (0,3) -- (2.3,3.1) -- (3,1.8)-- (4.2,0) -- (3.3, -1.7) -- (2.3, -3.0) -- (1,-3)-- (-1,-3.3) -- (-2.5,-2) -- (-3.8,0) -- (-2.8,1.8) -- (-2,3) -- (0,3);
\draw (1,2) -- (2.3,3.1); 
\draw (3,0) -- (4.2,0); 
\draw (2,-1.5) -- (2.3,-3.0);
\draw (-1,-2) -- (-1,-3.3); 
\draw (-2,0) -- (-3.8,0); 
\draw (-1,2) -- (-2,3); 
\end{scope}

\begin{scope}[xshift=12cm]
\draw[line width=1.3pt, color=red] (-1,2) -- (1,2) -- (3,0) -- (2, -1.5) -- (-1,-2) -- (-2,0) -- (-1,2);
\draw (-1,2) -- (0,3) -- (1,2) -- (3,1.8)-- (3,0) -- (3.3, -1.7) -- (2, -1.5) -- (1,-3)-- (-1,-2) -- (-2.5,-2) -- (-2,0) -- (-2.8,1.8) -- (-1,2);

\draw (0,3) -- (2.3,3.1) -- (3,1.8)-- (4.2,0) -- (3.3, -1.7) -- (2.3, -3.0) -- (1,-3)-- (-1,-3.3) -- (-2.5,-2) -- (-3.8,0) -- (-2.8,1.8) -- (-2,3) -- (0,3);
\draw (1,2) -- (2.3,3.1); 
\draw (3,0) -- (4.2,0); 
\draw (2,-1.5) -- (2.3,-3.0);
\draw (-1,-2) -- (-1,-3.3); 
\draw (-2,0) -- (-3.8,0); 
\draw (-1,2) -- (-2,3); 
\end{scope}
\end{tikzpicture}
\end{center}
\caption{Left: Part of the surface near a temporary hole. Free edges are drawn with red color. Right: Surface near the hole after deletion of simplices with more than one free edge.}
\label{fig:collision01}
\end{figure}
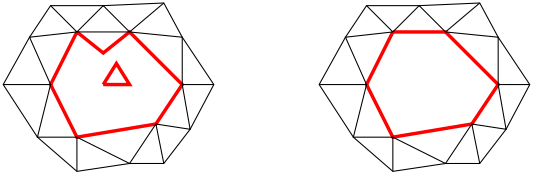

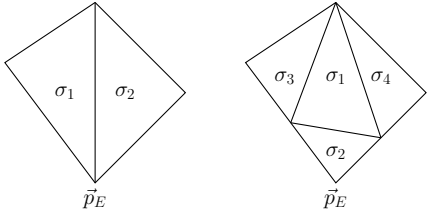
\begin{figure}[t]
\begin{center}
\begin{tikzpicture}[scale=0.4,transform shape]
\begin{scope}
\draw (0,1) -- (-3,-1) -- (0,-5) -- (3,-2) -- (0,1);
\draw (0,1) -- (0,-5); 
\draw (-1,-2) node{\huge $\sigma_1$};
\draw (1,-2) node{\huge $\sigma_2$};
\draw (0,-5.5) node{\huge $\vec p_E$};
\end{scope}

\begin{scope}[xshift=8cm]
\draw (0,1) -- (-3,-1) -- (0,-5) -- (3,-2) -- (0,1);
\draw (-1.5,-3) -- (1.5,-3.5); 
\draw (-1.5,-3) -- (0,1); 
\draw (0,1) -- (1.5,-3.5); 
\draw (0,-1.5) node{\huge $\sigma_1$};
\draw (0,-4) node{\huge $\sigma_2$};
\draw (-1.7,-1.5) node{\huge $\sigma_3$};
\draw ( 1.5,-1.5) node{\huge $\sigma_4$};
\draw (0,-5.5) node{\huge $\vec p_E$};
\end{scope}
\end{tikzpicture}
\end{center}
\caption{Improving mesh quality after a splitting or decrease of genus - Decrease the number of edges at the node $\vec p_E$.}
\label{fig:4sigma}
\end{figure}

\subsubsection{Algorithm for Merging and Increasing of Genus}
We propose the following algorithm for a possible modification of the surface triangulation after the detection of a merging or increase of genus.  
\begin{itemize}
\item \textbf{Preparation and deletion of simplices:} In case of merging or increasing genus, we consider the set of affected nodes $\vec X_j$, $j\in S$, which are located in the cube or in a neighbor cube, where the topology change has been detected. We delete all simplices with at least one vertex belonging to the set $\{\vec X_j\,:\, j \in S\}$.  This generates temporary free edges, i.e. simplices exist which do not have a neighbor simplex at that edge. The neighbor index corresponding to the free edge  is set to $-1$. This creates two intermediate holes. Simplices with two or three free edges are deleted as well, see Figure~\ref{fig:collision01}. 
\item \textbf{Matching free nodes/edges:} There exist now two sets of connected free edges. Let $I_{free,k}$, $k=1,2$, be the set of nodes corresponding to the free edges (end points of the edges). We try to match the nodes of $I_{free,1}$ with the nodes of $I_{free,2}$ using the Hungarian method \cite{Kuhn55} which is a combinatorial optimization algorithm. The Euclidean distance is used as cost criterion for matching two nodes. Since the number of nodes of the two sets need not be equal, there can be nodes which could not have been matched in the first step, see Figure~\ref{fig:collision02} (top). Therefore, new nodes are created by bisection of simplices at a free edge. Finally, each node can be matched, see Figure~\ref{fig:collision02} (bottom). 
\item \textbf{Point/Edge merging:} Since two matched nodes can have slightly different coordinates, they are replaced by one node in the middle of the line connecting the two nodes. The free edges of the two open holes are merged by identifying the matched nodes. The simplices, nodes and edges administration needs to be adapted. The nodes in $I_{free,1}$ are updated; each point is replaced by the mid point between it and its matching partner.  The nodes belonging to $I_{free,2}$ are deleted. Half of the free edges are deleted. The edge, node and neighbor information of the simplices at the former holes need to be adapted.  In case of merging, the surface index of all simplices belonging to the second surface is set to the surface index of the first surface. 
\end{itemize}

\begin{figure}[t]
\begin{center}
\begin{tikzpicture}[scale=0.6,transform shape]
\begin{scope}
\draw (-3,0) --   (-1,-0.2) -- (0.8,0.2) -- (2,-0.1) -- (3,0.3) -- (5,0); 
\draw (-3,-2) -- (-1.1,-2.2) --           (1.8,-1.8)          -- (5.2,-2.3); 

\draw[dashed] (-4,-0.2) -- (-3,0); 
\draw[dashed] (5,0) -- (6,0.3); 
\draw[dashed] (-4,-1.8) -- (-3,-2);
\draw[dashed] (5.2,-2.3) -- (6,-2); 

\draw[line width=1.3pt,color = red] (-3,0) -- (-3,-2); 
\draw[line width=1.3pt,color = red] (-1,-0.2) -- (-1.1,-2.2); 
\draw[line width=1.3pt,color = red] (2,-0.1) -- (1.8,-1.8); 
\draw[line width=1.3pt,color = red] (5,0) -- (5.2,-2.3); 

\tikzset{punkt/.style={fill,circle,inner sep=2.5pt,label=#1}}
\node[punkt]  at (-3,0) {} ;
\node[punkt]  at (-1,-0.2) {} ;
\node[punkt]  at (0.8,0.2) {} ;
\node[punkt]  at (2,-0.1) {} ;
\node[punkt]  at (3,0.3) {} ;
\node[punkt]  at (5,0) {} ;

\node[punkt]  at (-3,-2) {} ;
\node[punkt]  at (-1.1,-2.2) {} ;
\node[punkt]  at (1.8,-1.8) {} ;
\node[punkt]  at (5.2,-2.3) {} ;
\end{scope}

\begin{scope}[yshift = -4cm]
\draw (-3,0) --   (-1,-0.2) -- (0.8,0.2) -- (2,-0.1) -- (3,0.3) -- (5,0); 
\draw (-3,-2) -- (-1.1,-2.2) -- (0.35,-2.0) -- (1.8,-1.8) -- (3.5,-2.05)-- (5.2,-2.3); 

\draw[dashed] (-4,-0.2) -- (-3,0); 
\draw[dashed] (5,0) -- (6,0.3); 
\draw[dashed] (-4,-1.8) -- (-3,-2);
\draw[dashed] (5.2,-2.3) -- (6,-2); 

\draw[line width=1.3pt,color = red] (-3,0) -- (-3,-2); 
\draw[line width=1.3pt,color = red] (-1,-0.2) -- (-1.1,-2.2); 
\draw[line width=1.3pt,color = red] (2,-0.1) -- (1.8,-1.8); 
\draw[line width=1.3pt,color = red] (5,0) -- (5.2,-2.3); 

\draw[line width=1.3pt,color = green] (0.8,0.2) -- (0.35,-2.0); 
\draw[line width=1.3pt,color = green] (3,0.3) -- (3.5,-2.05);

\tikzset{punkt/.style={fill,circle,inner sep=2.5pt,label=#1}}
\node[punkt]  at (-3,0) {} ;
\node[punkt]  at (-1,-0.2) {} ;
\node[punkt]  at (0.8,0.2) {} ;
\node[punkt]  at (2,-0.1) {} ;
\node[punkt]  at (3,0.3) {} ;
\node[punkt]  at (5,0) {} ;

\node[punkt]  at (-3,-2) {} ;
\node[punkt]  at (-1.1,-2.2) {} ;
\node[punkt]  at (1.8,-1.8) {} ;
\node[punkt]  at (5.2,-2.3) {} ;

\node[punkt] at (0.35,-2.0) {}; 
\node[punkt] at (3.5,-2.05) {}; 

\end{scope}

\end{tikzpicture}
\end{center}
\caption{A subset of the free nodes. Top: Intermediate matching (red lines mark matching pairs). Bottom: Complete matching after inserting new nodes (red and green lines mark matching pairs).}
\label{fig:collision02}
\end{figure}
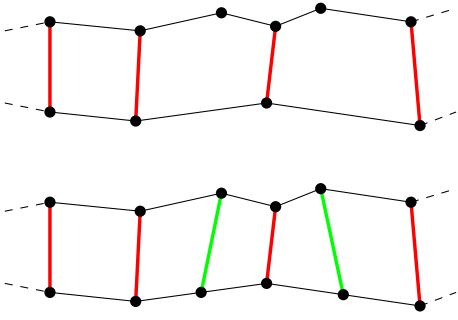

Note, that local refinement after a merging is typically necessary. This is automatically done, by a refinement method described in Section~\ref{subsec:additional_cormpuational_aspects_3d}. If the surface grows locally near the former merging part, the simplices will become greater compared to the average simplex of the surface. In this case, the large simplices will be refined.

The idea of creating two intermediate open holes and merging the two surfaces there is based on \cite{Brochu09}. There however, each hole is restricted to consist of exactly four free edges. 
New triangles between the free edges are created instead of merging the edges. 

In our method, the seeking for close points (and therefore close edges/simplices) is very efficient, since we make use of a background grid motivated by the method presented in \cite{MikulaUrban12}. We extended this method originally intended for curves in the plane to topology changes of surfaces. We allow for intermediate holes with an arbitrary number of free edges. The hole size is of the magnitude of the grid size. 

%In \cite{Brochu09} a variety of additional topological operations are proposed like edge operations (edge split, edge flip) and operations to avoid simplex collision. 

%The operations we have implemented beyond the detection of topology changes are described in the next sections.  Mesh operations are only needed in special situations, for example shortly before or after topology changes. In many situations, our method automatically provides a good mesh quality.

\subsection{Additional Computational Aspects}
\label{subsec:additional_cormpuational_aspects_3d}

\subsubsection{Computations of regions and coefficients}
\label{subsubsec:comp_regions_coeffs}
The computation of regions $\Omega_k^m$ and coefficients $c_k^m$, the mean of $u_0$ in $\Omega_k^m$,  is done as follows: We assign each voxel of the three-dimensional image domain to a phase $\Omega_k^m$. If a voxel is truncated by a surface, it is assigned to the phase to which the largest part belongs or to any of the two regions in case of two equal parts. Let $S_k^m$ be the set of $n_k^m$ voxels belonging to $\Omega_k^m$. Then the approximation $c_k^m$ is set to 
\begin{equation}
c_k^m := \frac{C_k^m}{n_k^m}, \quad\quad C_k^m := \sum_{vox \in S_k^m} u_0|_{vox}.
\end{equation}

The entire image domain needs to be considered only for $m=0$. For $m>0$, we only locally update the regions and re-compute the coefficients on this basis. For that, we consider a small band/tube of voxels around the current surfaces and look for changes of the region assignment. 

As the normal $\vec\nu_i^m$ points from $\Omega_{k^-(i)}^m$ to $\Omega_{k^+(i)}^m$, the voxels close to the surface $\Gamma_i^m$ can be assigned to the phase $k^+(i)$ or $k^-(i)$, respectively. 

In the update step, we first set $n_k^m = n_k^{m-1}$ and $C_k^m = C_k^{m-1}$ for $k=1,\ldots,N_R$. For $i=1,\ldots,N_S$, all voxels in an environment of $\Gamma_i^m$ are subsequently considered. Let a voxel $vox$ be assigned to phase $k \in \{k^+(i), k^-(i)\}$ and let $l\neq k$ be the former phase index of the voxel. Then, we set
\begin{align}
n_k^m &= n_k^m + 1, \quad n_l^m = n_l^m -1, \nonumber \\
C_k^m &= C_k^m + u_0|_{vox}, \quad C_l^m = C_l^m - u_0|_{vox}.
\end{align}
After having considered all voxels close to the surfaces, the coefficients are set to $c_k^m = C_k^m / n_k^m$ for $k=1, \ldots, N_R$. 

\subsubsection{Time Step Control}
\label{subsubsec:time_step_control_3D}
We use a certain adaptive time step setting to control the speed of the evolution of the surface(s). Let $\Delta t = \tau_m$ denote the (possibly variable) time step size. The time step size is controlled as follows: Let $\delta X_n^\mathrm{min}>0$, $\delta X_n^\mathrm{max}>0$ with $\delta X_n^\mathrm{min} < \delta X_n^\mathrm{max}$ be user-defined tolerances for the absolute value of the position difference in normal direction. Let $\Delta t > 0$ be an initial time step size for $m=0$ or the time step size of the previous time step for $m > 0$. 

We propose the following time step size control:
\noindent Choose a factor ${\lambda_t}\in\mathbb N$ (for example $\lambda_t=2$ or $\lambda_t=10$). 
\begin{enumerate}
\item Solve equation \eqref{eq:schur2_3d} and set $\delta X_{n}^{m+1}$ to the maximum of $|\delta \vec X_{i,j}^{m+1} \,.\, \vec\omega_{i,j}^m|$ for $i=1,\ldots,N_S$ and $j=1, \ldots,$ $N_{i,V}$. 
\item If $\delta X_{n}^{m+1} > \delta X_n^\mathrm{max}$, set $\Delta t$ to $\frac1{\lambda_t} \Delta t$ and repeat step~(i). 
\item Otherwise, if $\delta X_{n}^{m+1} < \delta X_n^\mathrm{min}$, set $\Delta t$ to ${\lambda_t} \Delta t$ and repeat step~(i).
\item Otherwise, proceed by checking for topology changes (see above) and go to the next time step, i.e. set $m$ to $m+1$. 
\end{enumerate}
The effect of this time step size control is simple: If there are too high changes in the position of the nodes in normal direction (i.e. if the normal velocity is too high), the time step size will be decreased. This occurs if the sum of weighted curvature and external term is high. If the change in the position in normal direction is too small, the time step size will be increased to speed up the image segmentation process.

\subsubsection{Mesh Quality Aspects}
\label{subsubsec:mesh_quality}
During the evolution of surface, it may be necessary to control the mesh quality. For example, if a surface continuously grows, the simplices become larger and should be refined if their area exceeds a certain threshold. Similarly, too small simplices should be deleted. 

For computing the matrix entries, cf. \eqref{eq:def_matrices_MNA_3d}, we already need to compute the area of each simplex of the triangulation of the surface $\Gamma_i^m$, $i \in \{1, \ldots, N_S\}$. Let $A_\mathrm{desired}>0$ be a predefined desired area for one simplex.  Let $a > 0$ be a given factor (e.g. $a=2$ or $a=10$).  If the area of a simplex exceeds $a A_\mathrm{desired}$, it will be refined by bisection of its largest edge. Its neighbor simplex across the refinement edge will be also refined such that no hanging nodes remain. 

Local refinement is necessary to avoid too large simplices. Further, a mesh can also be continuously refined, for example when one starts with a small surface which globally grows. The triangles of the growing surface are refined one by one. 
Furthermore, the triangles which have one very large angle, i.e. an angle larger than a given threshold (e.g. $\geq 160^\circ$), are also refined. 

If a simplex area is smaller than a certain percentage of the desired area $A_\mathrm{desired}$, for example smaller than $1\%$,  the simplex is deleted. Further, it is also deleted if one of its three inner angles is smaller than a given threshold, for example smaller than $2^\circ$. When a simplex is marked for deletion, one or more neighbor simplices are also deleted. 

Mesh operations like deletion of triangles are rarely necessary. These operations are usually performed only a few times, for example close before or after topology changes. 

Figure \ref{fig:refinement_deletion_simplices} illustrates examples how the triangulation is adapted close to simplices which are marked for refinement or deletion. 

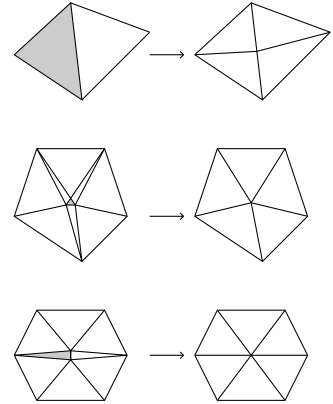
\begin{figure}
\begin{center}
\begin{tikzpicture}[scale=0.3,transform shape]
\draw[white] (-2,2) rectangle (15,5);

\filldraw[fill = black!20] (-1,0) -- (2, -2) -- (1.5,2.3) -- (-1,0); 
\draw (5,1) -- (2, -2);
\draw (1.5,2.3) -- (5,1); 

\draw (5,0) -- (6.5, 0);
\draw (6.35,0.15) -- (6.5, 0);
\draw (6.35,-0.15) -- (6.5, 0);

\draw (7,0) -- (10, -2) -- (9.5,2.3) -- (7,0); 
\draw (13,1) -- (10, -2); 
\draw (9.5,2.3) -- (13,1); 
\draw (7,0) -- (9.75, 0.15) -- (13,1); 

\end{tikzpicture}

\begin{tikzpicture}[scale=0.3,transform shape]
\draw[white] (-2,2) rectangle (15,5);

\draw (0,3) -- (3,3) -- (4,0) -- (2,-2) -- (-1,0) -- (0,3); 
\filldraw[fill = black!20]  (1.3,0.5) -- (1.5,0.8) -- (1.7,0.5) -- (1.3,0.5); 
\draw (0,3) -- (1.5,0.8) -- (3,3); 
\draw (0,3) -- (1.3,0.5) -- (-1,0); 
\draw (1.3,0.5) -- (2,-2) -- (1.7,0.5); 
\draw (4,0) -- (1.7, 0.5) -- (3,3);  

\draw (5,0) -- (6.5, 0);
\draw (6.35,0.15) -- (6.5, 0);
\draw (6.35,-0.15) -- (6.5, 0);

\draw (8,3) -- (11,3) -- (12,0) -- (10,-2) -- (7,0) -- (8,3); 
(9.5,0.6)
\draw (8,3) -- (9.5,0.6) -- (11,3); 
\draw (9.5,0.6) -- (7,0); 
\draw (9.5,0.6) -- (10,-2); 
\draw (12,0) -- (9.5,0.6);  

\end{tikzpicture}

\begin{tikzpicture}[scale=0.3,transform shape]
\draw[white] (-2,2) rectangle (15,5);

\filldraw[fill = black!20] (-1,1) -- (1.5,0.8) -- (1.5,1.2) -- (-1,1); 
\draw (4,1) -- (1.5,0.8) -- (1.5,1.2) -- (4,1);
\draw (-1,1) -- (0,-1) -- (3,-1) -- (4,1) -- (3,3) -- (0,3) -- (-1,1);
\draw (0,3) -- (1.5,1.2) -- (3,3); 
\draw (0,-1) -- (1.5,0.8) -- (3,-1); 

\draw (5,1) -- (6.5, 1);
\draw (6.35,1.15) -- (6.5, 1);
\draw (6.35,0.85) -- (6.5, 1);

\draw (7,1) -- (8,-1) -- (11,-1) -- (12,1) -- (11,3) -- (8,3) -- (7,1);
\draw (8,3) -- (9.5,1) -- (11,3); 
\draw (8,-1) -- (9.5,1) -- (11,-1); 
\draw (7,1) --(9.5,1) -- (12,1);

\end{tikzpicture}
\end{center}
\caption{Top: Refinement of a simplex (marked in gray). The neighbor simplex is also refined to avoid hanging nodes. Center: Deletion of a simplex with a too small area (marked in gray). Three neighbor simplices are also deleted. Bottom: Deletion of a simplex with a too small angle (marked in gray). One neighbor simplex is also deleted.}
\label{fig:refinement_deletion_simplices}
\end{figure}

Our numerical method for surface evolution for image segmentation tasks is based on a numerical method developed in \cite{BGN08c}. This method provides a good mesh quality in many cases. However, if, for example, a pinch-off occurs, the mesh can get distorted and some routine for keeping a good mesh quality is needed. 

The idea of a mesh regularization method proposed in \cite{BGN08d} is to induce or reduce the tangential motion of nodes along a surface. We use this method to control the tangential motion of nodes of surfaces during 3D image segmentation. For details, we refer to \cite{BGN08d}. The system \eqref{eq:fem_scheme_3d} is replaced by a scheme which controls also the tangential motion of the nodes. 

\subsection{Summary of the Image Segmentation Algorithm}
In summing up, we propose the following algorithm for segmentation of 3D images. Given a set of triangulated surfaces $\Gamma^0 = (\Gamma_1^0, \ldots, \Gamma_{N_S}^0)$ and nodes $\vec X_{i,j}^0$, $i=1,\ldots,N_S$, $j=1, \ldots, N_{i,V}$, perform the following steps for $m=0, 1, \ldots, M-1$: 
\begin{enumerate}
\item \label{step1_3d} Compute the regions $\Omega_k^m$ and the coefficients $c_k^m$, $k=1,\ldots,N_R$, as described in Section \ref{subsubsec:comp_regions_coeffs}. 
\item \label{step2_3d} Compute $b^m$ as defined in \eqref{eq:def_b_m_3d} by using the coefficients $c_k^m$ of step~\ref{step1_3d}. Compute $\vec X^{m+1} = \vec X^m + \delta \vec X^{m+1}$ by solving the linear equation \eqref{eq:schur2_3d}.
\item Check whether the time step size needs to be increased or decreased, see Section~\ref{subsubsec:time_step_control_3D}. If the time step size needs to be changed, repeat step~\ref{step2_3d} with the new time step size. 
\item Check whether topology changes occur and execute the topology change, see Section~\ref{subsec:top_changes_3d}. 
\item If necessary, refine too large simplices or delete too small simplices of the triangulation as described in Section~\ref{subsubsec:mesh_quality}. 
\end{enumerate}

%% file: 4_results.tex
\section{Results}
\label{sec:results}

\subsection{Artificial Test Images}
In this section we demonstrate the developed method for segmentation of 3D images with parametric active surfaces. We first study four examples of artificial test images to demonstrate the ability of the method to detect different topology changes (splitting, merging, increase and decrease of genus).

In the first experiment, we demonstrate how a surface is split in two surfaces. We consider an artificial image defined on an image domain given given by the cuboid $\Omega = [-2.5, 2.5] \times [-1.5, 1.5] \times [-1.5, 1.5]$. The image intensity function is defined by 
\begin{equation*}
u_0: \Omega \rightarrow \mathbb{R}, \quad u_0(\vec x) = \left\{
\begin{array}{ll}
0 & \text{if } \,\, \| \vec x - (-1.2,0,0)^T \| \leq 0.8  \\
  &  \,\vee\, \| \vec x - (1.2,0,0)^T \| \leq 0.8, \\[2ex]
1 & \text{else}.
\end{array}\right.
\end{equation*}
The three-dimensional image contains two balls centered at $(\pm 1.2, 0, 0)^T \in \mathbb{R}^3$ with radius $0.8$. The segmentation process is started using a cylinder-like surface as initial surface placed in the center of the cuboid. Figure~\ref{fig:splitting_im_seg} shows the surface at different time steps. 

For weighting the curvature term and the forcing term for the image segmentation, the parameters $\sigma = 1$ and $\lambda = 100$ are used. At time step $m=214$ a splitting of the evolving surface occurs. In the subsequent iterations steps, the two new surfaces each evolve to a ball. To detect the topology change, we use an auxiliary background grid with grid size $a=0.025$ as described in Section \ref{subsubsec:detection_topology_change_3D}. A cube of the grid is considered for possible topology changes if more than $N_\mathrm{detect}=10$ nodes are located inside the cube. We further use the parameters $thr1=30^\circ$, $thr2=150^\circ$ and  $thr3=40^\circ$.

\begin{figure}
	\centering
		\includegraphics[trim = 10mm 20mm 10mm 20mm,clip, width = 0.2\textwidth]{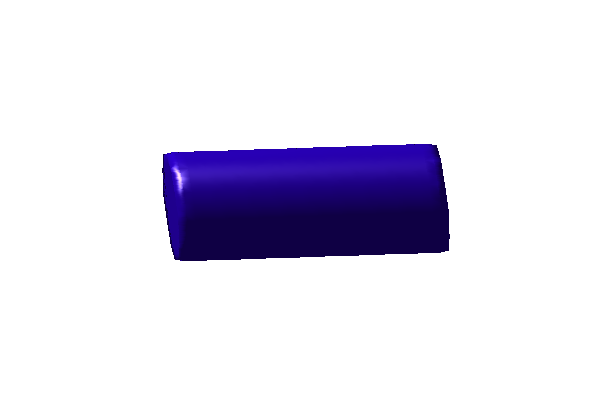}
		\includegraphics[trim = 10mm 20mm 10mm 20mm,clip, width = 0.2\textwidth]{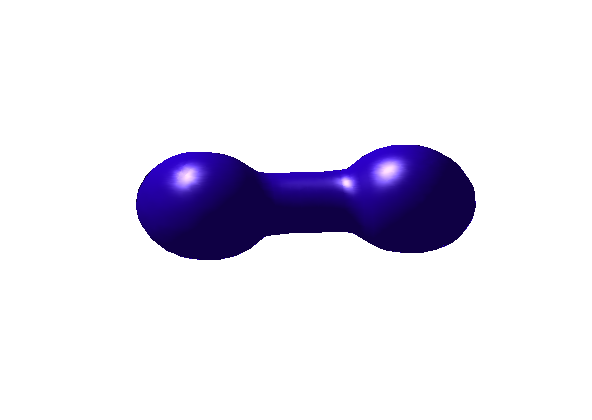}\\
		\includegraphics[trim = 10mm 20mm 10mm 20mm,clip, width = 0.2\textwidth]{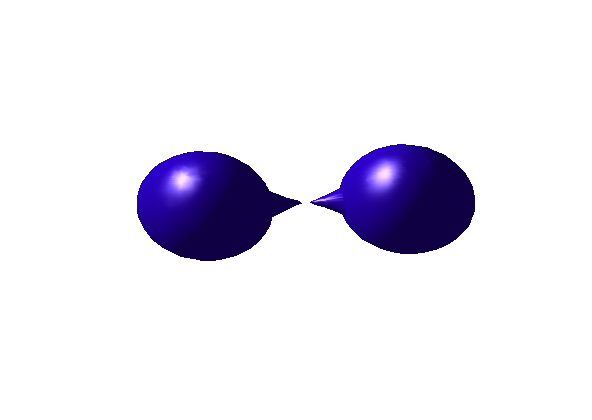}
		\includegraphics[trim = 10mm 20mm 10mm 20mm,clip, width = 0.2\textwidth]{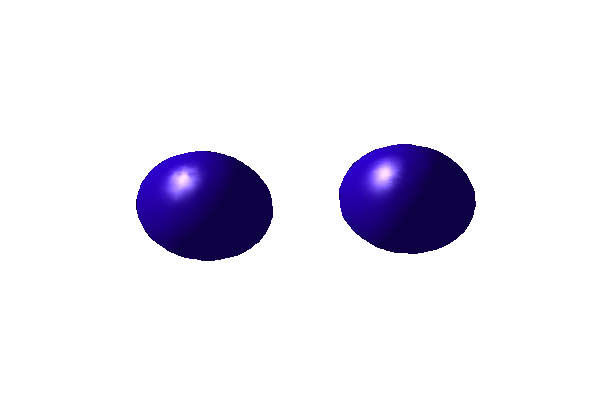}
	\caption{Splitting of a surface during 3D image segmentation. Surface(s) at step $m=0, 100, 215, 300$  at time $t_m = 0,  0.01,  0.0215, 0.0298$.}
	\label{fig:splitting_im_seg}
\end{figure}

\begin{figure}
	\centering
		\includegraphics[trim = 10mm 10mm 10mm 10mm,clip, width = 0.18\textwidth]{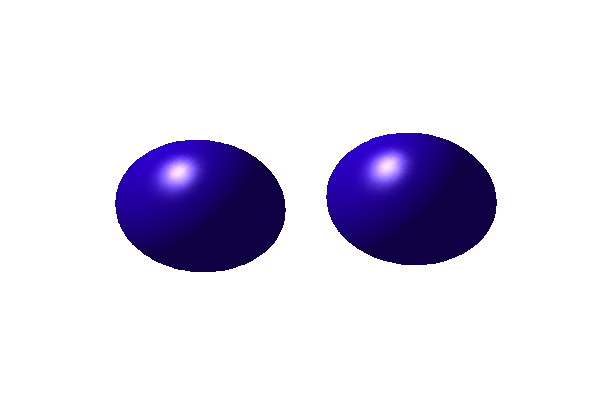}
		\includegraphics[trim = 10mm 10mm 10mm 10mm,clip, width = 0.18\textwidth]{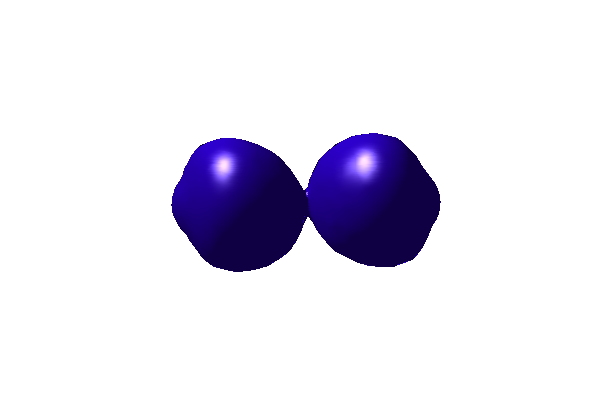}\\
		\includegraphics[trim = 10mm 10mm 10mm 10mm,clip, width = 0.18\textwidth]{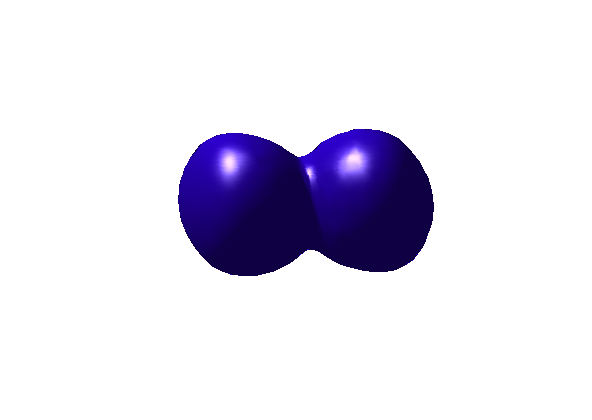}
		\includegraphics[trim = 10mm 10mm 10mm 10mm,clip, width = 0.18\textwidth]{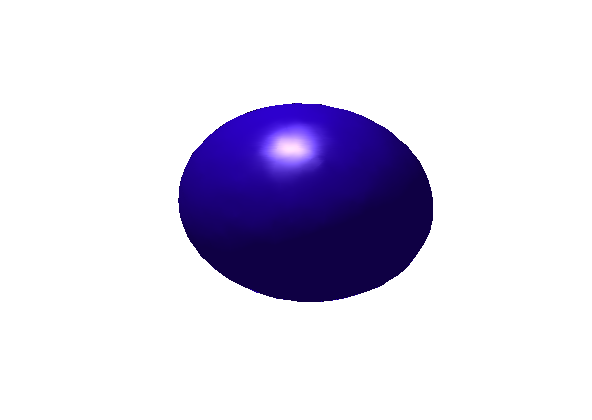}
	\caption{Merging demonstration. Surface(s) at step $m=0, 50, 80, 200$  at time $t_m = 0, 0.005, 0.008, 0.02$.}
	\label{fig:merging_im_seg}
\end{figure}

Furthermore, we also perform a time step control (cf. Section \ref{subsubsec:time_step_control_3D}) using the thresholds $\delta X_n^\mathrm{min} = 0.003$ and $\delta X_n^\mathrm{max} = 0.05$. For almost each time step, a time step size of $\Delta t = 10^{-4}$ was used. Only immediately after the splitting, for two time steps, $\Delta t$ was reduced to $10^{-5}$ to avoid a too fast retraction of the newly generated surfaces close to the former splitting point. 

The reversed topology change of splitting is a merging of two surfaces to one single surface. The initial surfaces in the next example are two balls. The image domain is given by $\Omega = [-1.2, 1.2] \times [-0.8, 0.8] \times [-0.8, 0.8]$ and the image intensity function is defined by 
\begin{equation*}
u_0: \Omega \rightarrow \mathbb{R}, \quad u_0(\vec x) = \left\{
\begin{array}{ll}
0 & \text{if } \,\, \| \vec x \| \leq 0.6, \\[2ex]
1 & \text{else}.
\end{array}\right.
\end{equation*}
As weighting parameters $\sigma = 2$ and $\lambda = 60$ are used. Time step control is performed applying the thresholds $\delta X_n^\mathrm{min} = 0.001$ and $\delta X_n^\mathrm{max} = 0.02$. No change of the time step size is necessary in this example; the time step size $\Delta t = 10^{-4}$ need not be changed throughout the evolution. To detect the merging, $a=0.03$, $N_\mathrm{detect}=10$ and $thr1=20^\circ$, $thr2=150^\circ$ and  $thr3=40^\circ$ are used. The resulting surfaces of this experiment at different time steps are shown in Figure \ref{fig:merging_im_seg}. 

Since the surface grows continuously, some simplices have to be refined as described in Section \ref{subsubsec:mesh_quality}. The desired area for one simplex is $A_\mathrm{desired}=0.001$; a simplex is refined by bisection of its largest angle if its area is larger than a certain factor of $A_\mathrm{desired}$. A simplex is also bisected if one angle is larger than $170^\circ$. A simplex is deleted if one angle is smaller than $2^\circ$ or if its area is smaller than $1\%$ of the desired area of $A_\mathrm{desired}$.

\begin{figure}
	\centering
		\includegraphics[width = 0.18\textwidth]{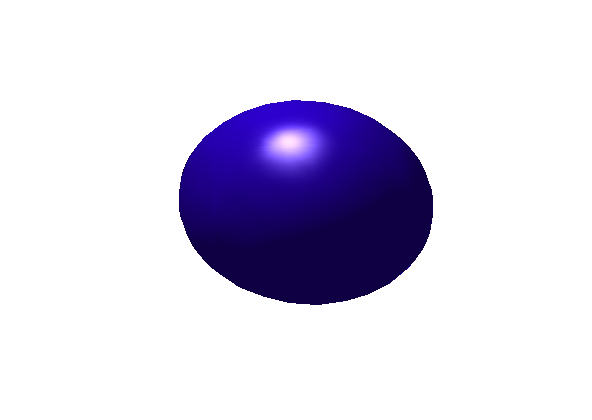}
		\includegraphics[width = 0.18\textwidth]{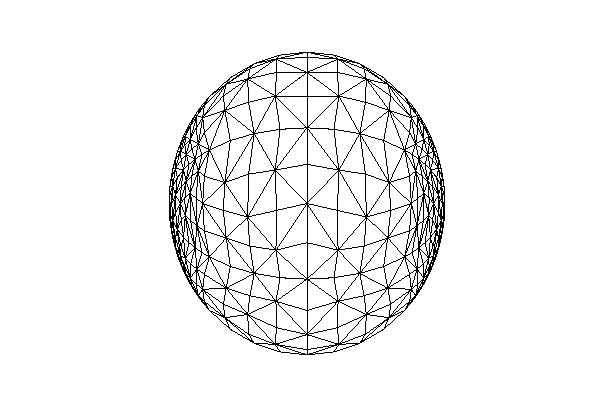}\\
		\includegraphics[width = 0.18\textwidth]{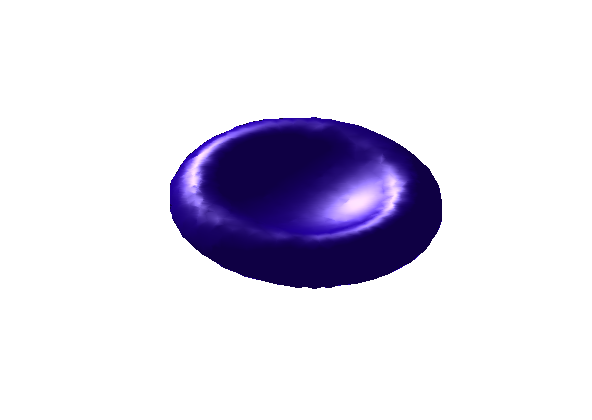}
		\includegraphics[width = 0.18\textwidth]{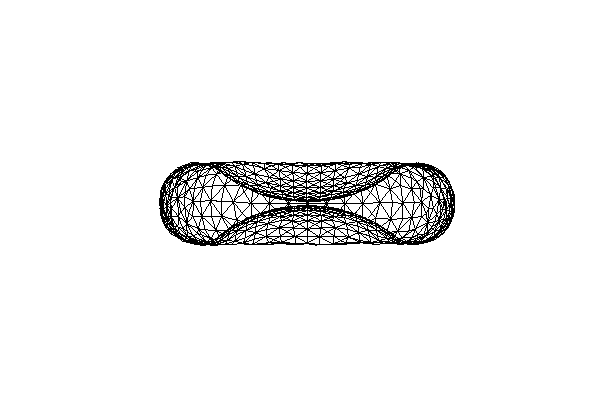}\\
		\includegraphics[width = 0.18\textwidth]{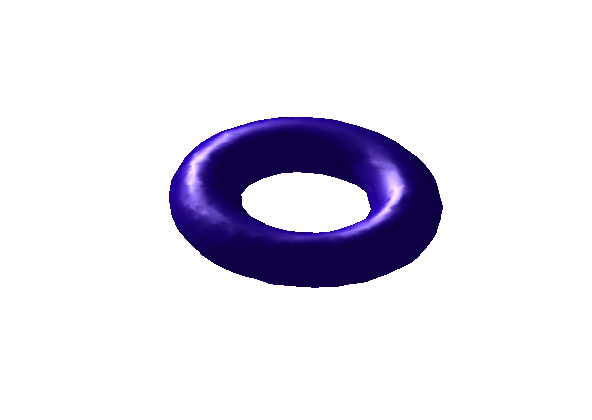}
		\includegraphics[width = 0.18\textwidth]{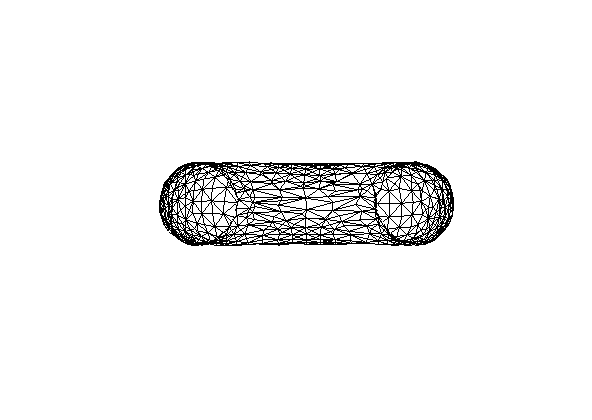}
	\caption{Demonstration of an increase of the genus of a surface. Surface, mesh and cross-section at step $m=0, 325, 500$ (row-wise) at time $t_m = 0, 0.325, 0.5$. Column 1-2: surface and mesh (cross-section).}
	\label{fig:increase_genus_im_seg}
\end{figure}

\begin{figure}
	\centering
		\includegraphics[trim = 30mm 10mm 30mm 10mm,clip, width = 0.12\textwidth]{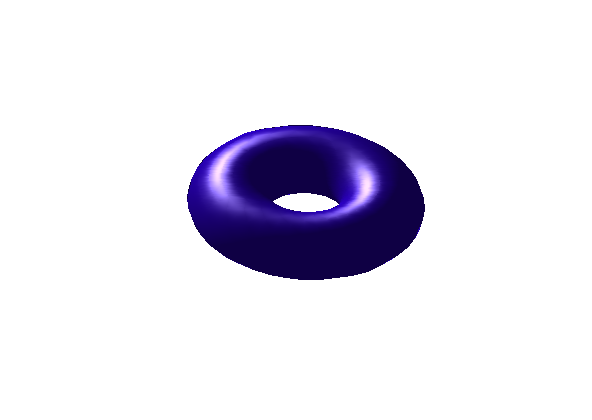}
		\includegraphics[trim = 30mm 10mm 30mm 10mm,clip, width = 0.12\textwidth]{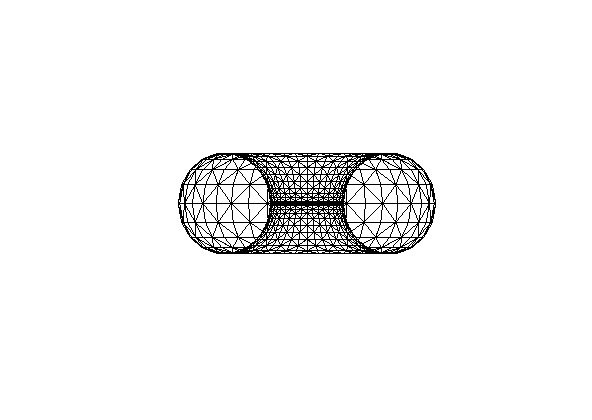}\\
		\includegraphics[trim = 30mm 10mm 30mm 10mm,clip, width = 0.12\textwidth]{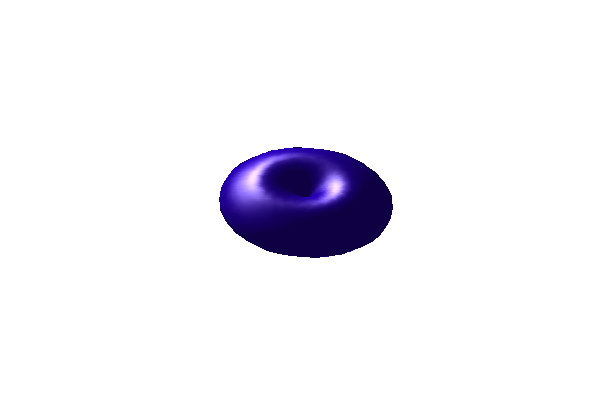}
		\includegraphics[trim = 30mm 10mm 30mm 10mm,clip, width = 0.12\textwidth]{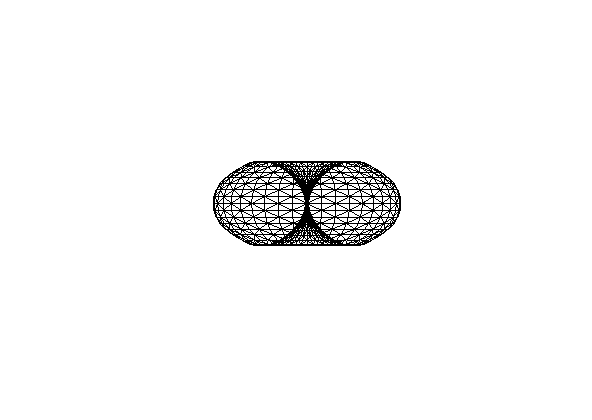}\\
		\includegraphics[trim = 30mm 10mm 30mm 10mm,clip, width = 0.12\textwidth]{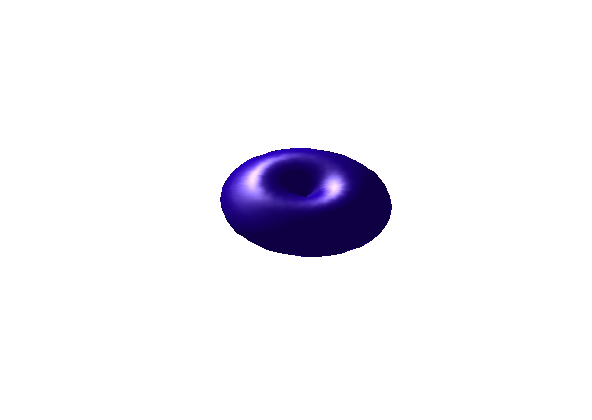}
		\includegraphics[trim = 30mm 10mm 30mm 10mm,clip, width = 0.12\textwidth]{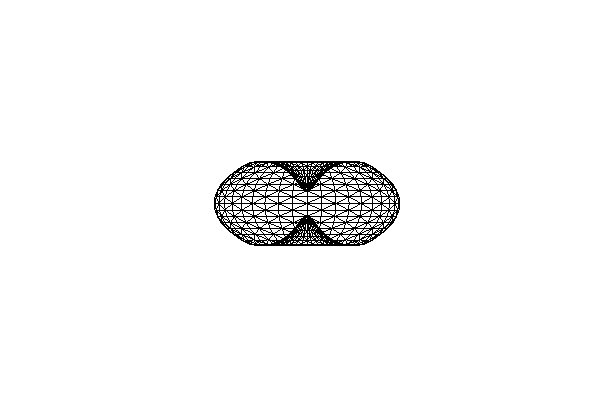}\\
		\includegraphics[trim = 30mm 10mm 30mm 10mm,clip, width = 0.12\textwidth]{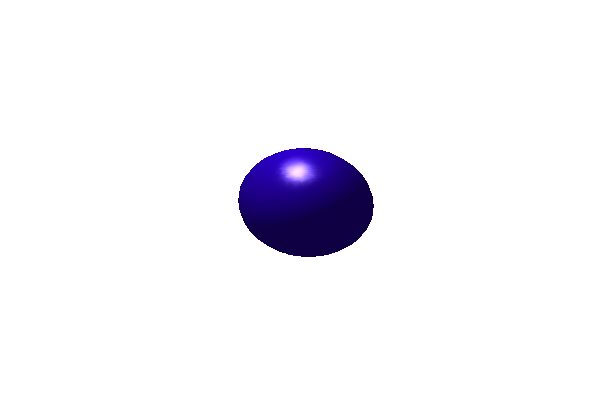}
		\includegraphics[trim = 30mm 10mm 30mm 10mm,clip, width = 0.12\textwidth]{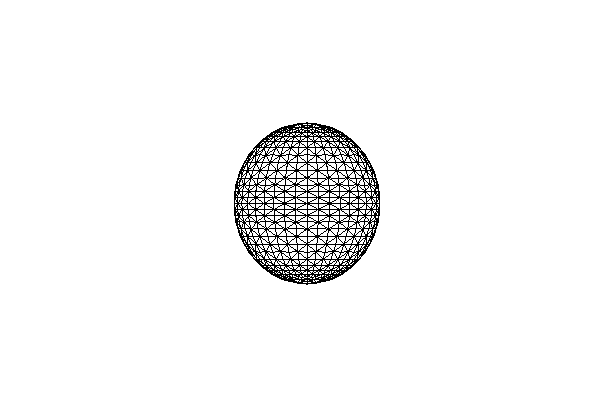}
	\caption{3D image segmentation example where a torus evolves to a ball. Surface at step $m=0, 425, 500, 1000$ (row-wise) at time $t_m = 0, 0.0425, 0.04406, 0.27316$. Column 1-2: surface and mesh (cross-section).}
	\label{fig:decrease_genus_im_seg_1}
\end{figure}

In the next examples, we demonstrate another kind of topology changes: increase and decrease of the genus of a surface. Therefore, we consider an image segmentation example where a sphere should evolve to a torus. The image intensity function is given by 
\begin{equation}
u_0(\vec x) = \left\{
\begin{array}{ll}
0 & \text{if } \,\, (\sqrt{x_1^2 + x_2^2}-R)^2 + x_3^2 \leq r^2,\\
1 & \text{else},
\end{array}
\right.
\end{equation}
where $R=1.2$ and $r=0.4$ are used here. 

Figure \ref{fig:increase_genus_im_seg} shows the surface, its mesh (cross-section of the mesh) at different time steps. For this example we apply $\sigma = 1$, $\lambda = 60$ (weighting parameters). The topology change is detected using $a=0.0565$, $N_\mathrm{detect}=8$ and $thr1=20^\circ$, $thr2=150^\circ$ and  $thr3=40^\circ$. As parameters to control the refinement, the desired triangle area is set to $A_\mathrm{desired}=0.005$, and the angles $170^\circ$ and $2^\circ$ are used for bisection or deletion of a triangle, respectively. 

%In this example, the initial mesh is generated by using a simple triangulated cube, followed by a projection of each vertex to the sphere. This results in an initial mesh with a relatively poor mesh quality, i.e. the vertices are not distributed equally over the surface. This example shows that our method can also cope with difficult meshes: Large triangles are bisected and our method  distributes the nodes over the surface during the evolution of the surface.

\begin{figure}
	\centering
		\includegraphics[viewport = 100 270 480 560, width = 0.35\textwidth]{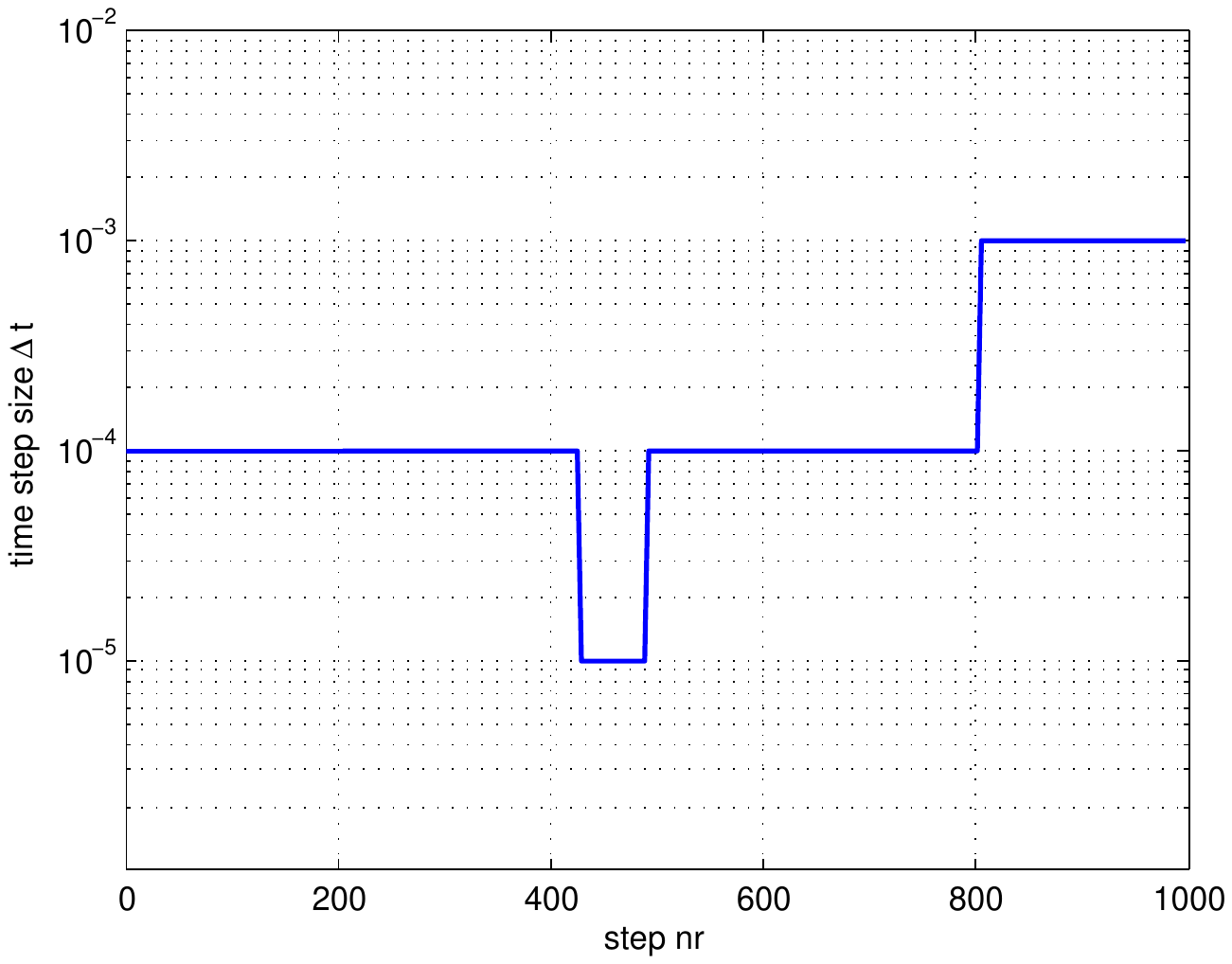}
	\caption{Time step sizes during the evolution of the torus to a ball.}
	\label{fig:decrease_genus_im_seg_3}
\end{figure}

Further, $\delta X_n^\mathrm{min} = 0.01$ and $\delta X_n^\mathrm{max} = 0.1$ are applied as thresholds for the time step size control. Throughout the evolution, there was no need to change the initial time-step size of $\Delta t=10^{-3}$. 

At time step $m=325$, two different parts (top and bottom) of the surface with nearly opposite normal vector nearly touch. A topology change is detected and a small hole occurs. The genus of the surface is increased from $g=0$ to $g=1$. At time step $m=500$, the 3D object, a torus, is detected; its boundary is represented by the surface.  

Finally, we present an example where a torus is used as initial surface and a sphere should be detected. The image intensity function is given by 
\begin{equation*}
u_0: \Omega \rightarrow \mathbb{R}, \quad u_0(\vec x) = \left\{
\begin{array}{ll}
0 & \text{if } \,\, \| \vec x \| \leq 0.8, \\
1 & \text{else}.
\end{array}\right.
\end{equation*}

Figure \ref{fig:decrease_genus_im_seg_1}  shows the surface at several time steps. As weighting parameters $\sigma=1$ and $\lambda=20$ are applied. 
For the detection of the decrease of genus, the parameters $a=0.025$, $N_\mathrm{detect}=20$ and $thr1=20^\circ$, $thr2=150^\circ$ and  $thr3=40^\circ$ are used. 

The time step size is controlled using the thresholds $\delta X_n^\mathrm{min} = 0.0005$ and $\delta X_n^\mathrm{max} = 0.01$. 
Figure \ref{fig:decrease_genus_im_seg_3} shows the time step sizes during the image segmentation process. After the topology change the time step size is decreased from $10^{-4}$ to $10^{-5}$. Later it is increased to speed up the segmentation.

\subsection{Segmentation of Medical 3D Images}

%------------------------------------ m = 0 -----------------------------------------------------------------%
\begin{figure}
	\centering
		\includegraphics[width = 0.28\textwidth]{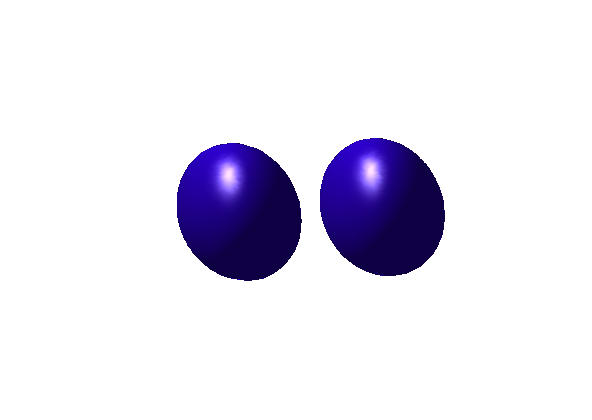}\\
		\includegraphics[viewport = 130 320 460 520, width = 0.15\textwidth]{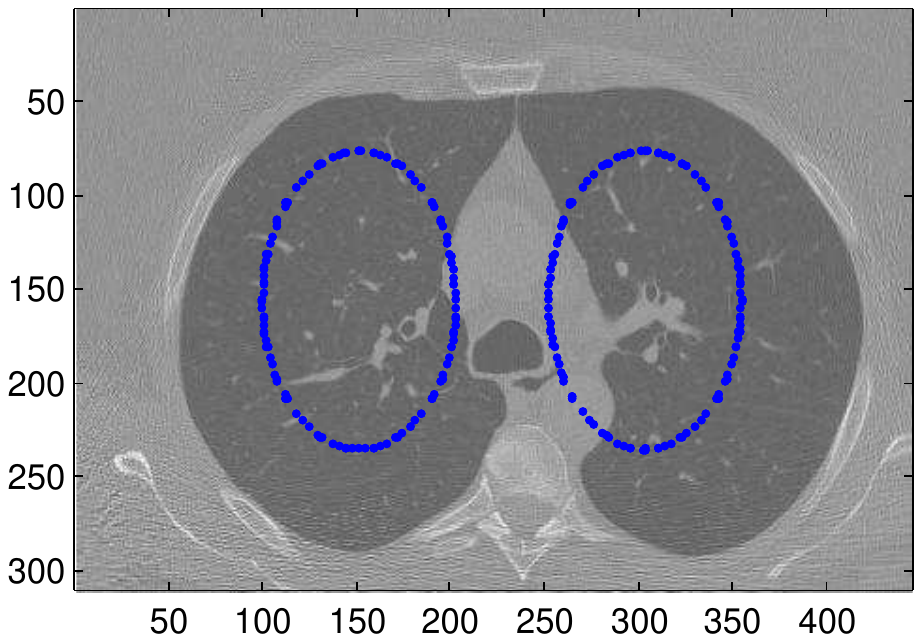}
		\includegraphics[viewport = 130 320 460 520, width = 0.15\textwidth]{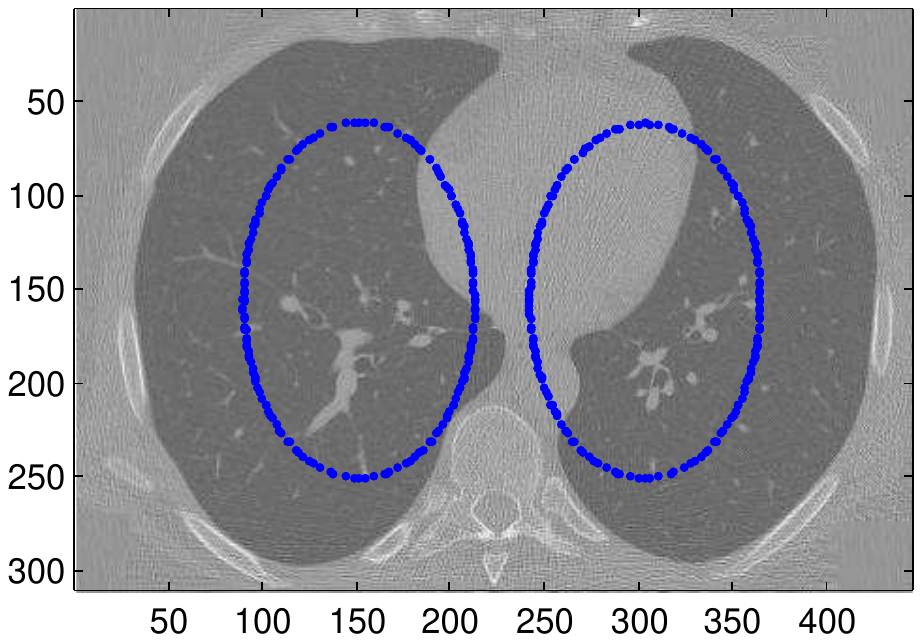}
		\includegraphics[viewport = 130 320 460 520, width = 0.15\textwidth]{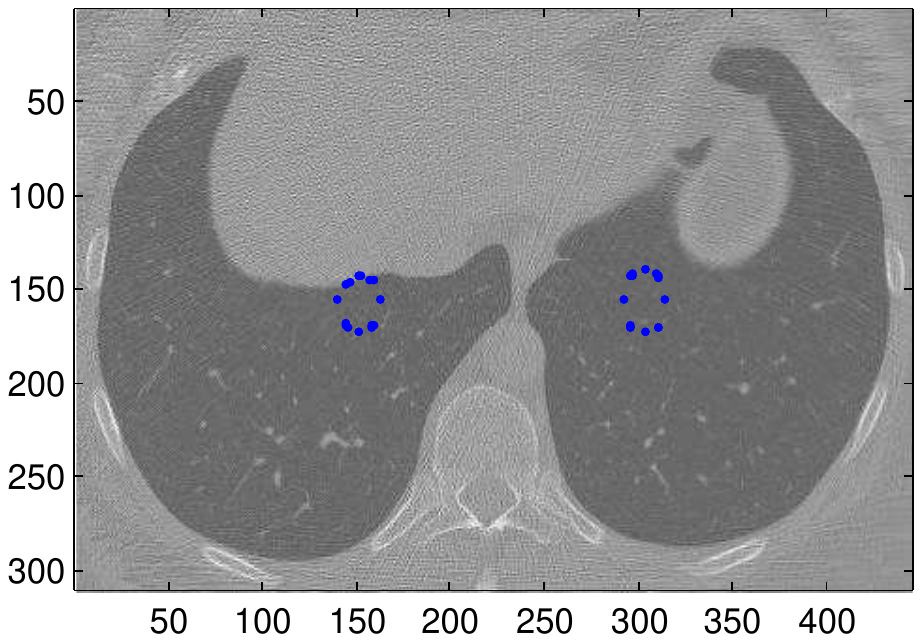}\\
		\includegraphics[viewport = 130 320 460 520, width = 0.15\textwidth]{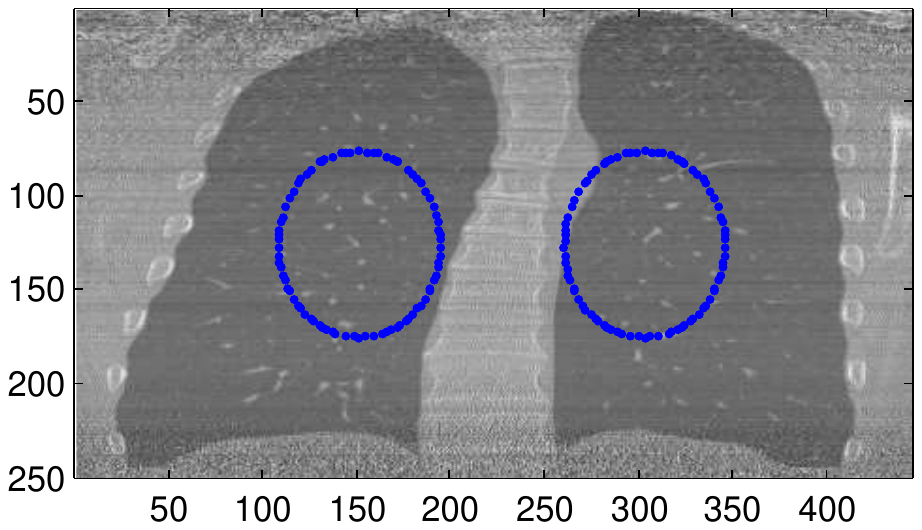}
		\includegraphics[viewport = 130 320 460 520, width = 0.15\textwidth]{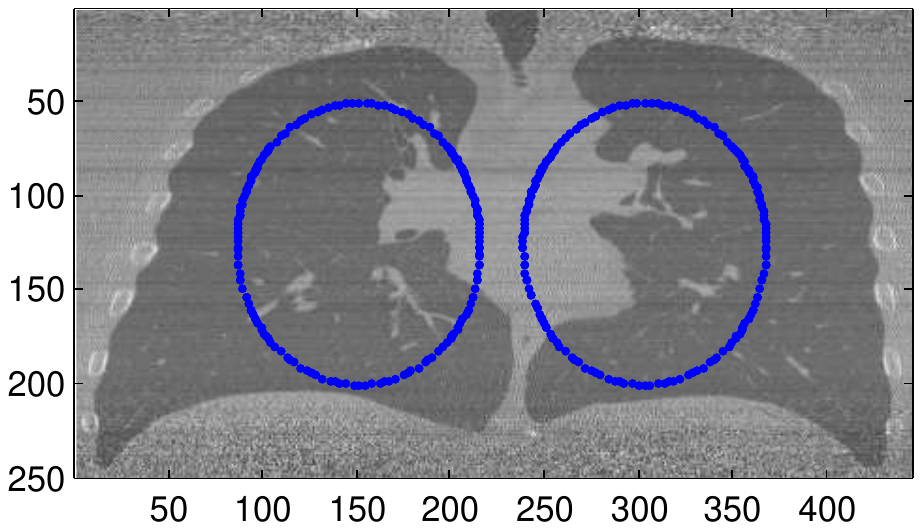} 
		\includegraphics[viewport = 130 320 460 520, width = 0.15\textwidth]{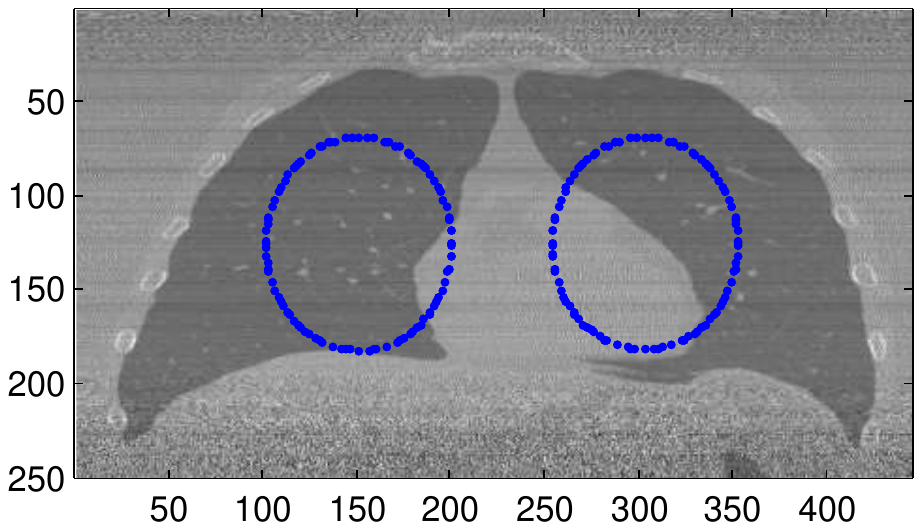} 
	\caption{Lung segmentation: Surfaces (row 1) and cross-sections (row 2: $z=80, 150, 200$, row 3: $y=80, 150, 200$) at $m=0$ at time $t=0$. The original images are from the Lung Image Database Consortium image collection (LIDC-IDRI) of The Cancer Imaging Archive (TCIA), see \cite{Reeves2007},\cite{Armato2011}, \cite{Reeves2011}..}
	\label{fig:lung_0001}
\end{figure}

%------------------------------------ m = 100 -----------------------------------------------------------------%
\begin{figure}
	\centering
		\includegraphics[width = 0.28\textwidth]{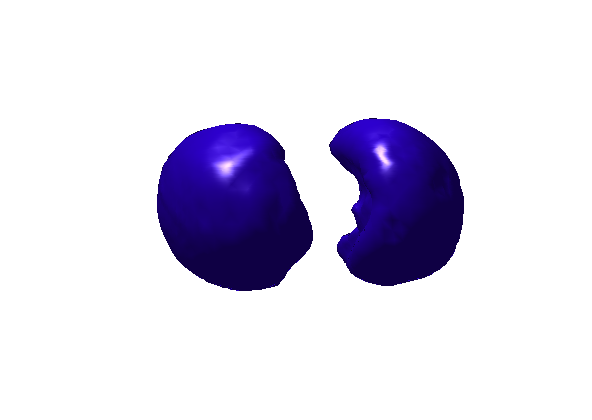}\\
		\includegraphics[viewport = 130 320 460 520, width = 0.15\textwidth]{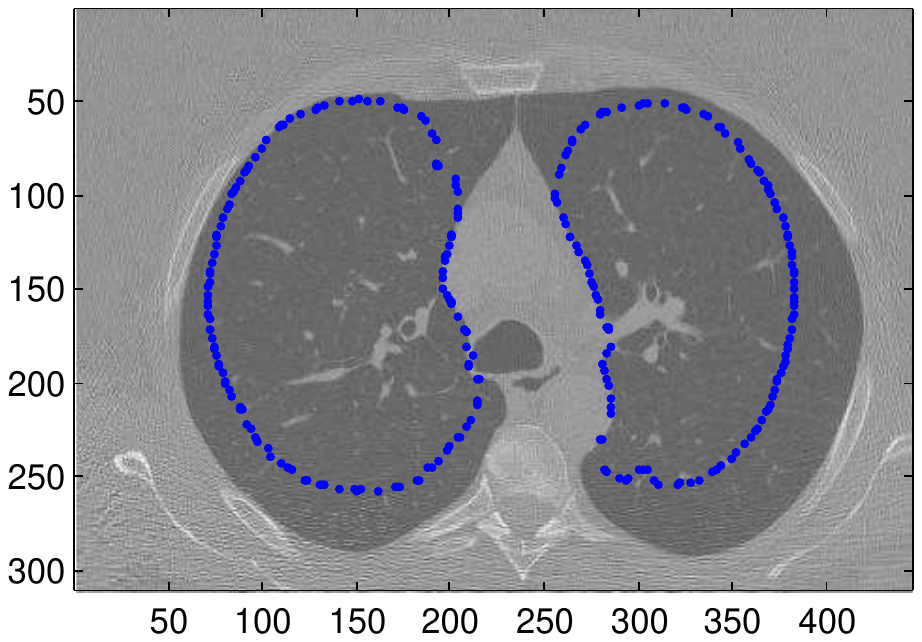}
		\includegraphics[viewport = 130 320 460 520, width = 0.15\textwidth]{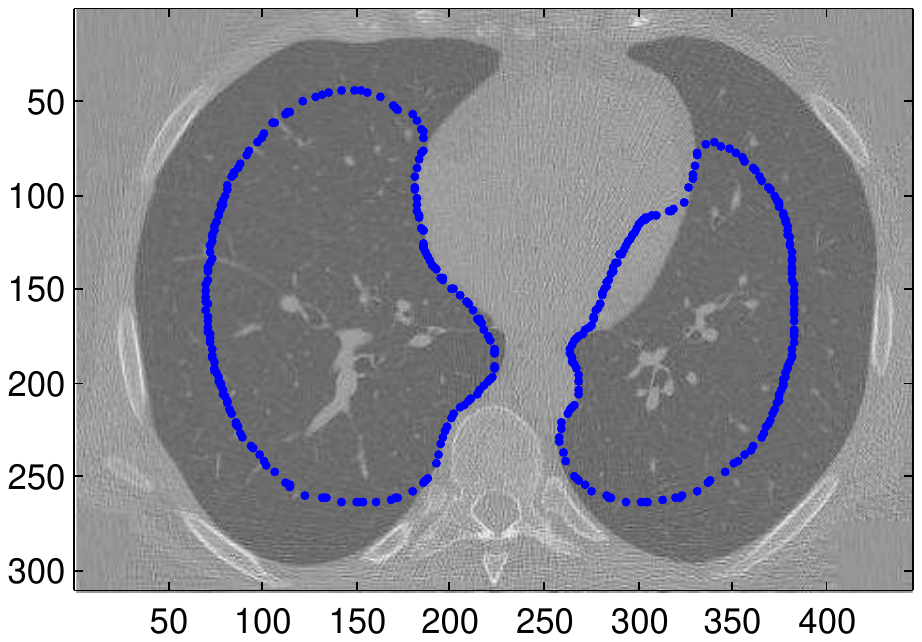}
		\includegraphics[viewport = 130 320 460 520, width = 0.15\textwidth]{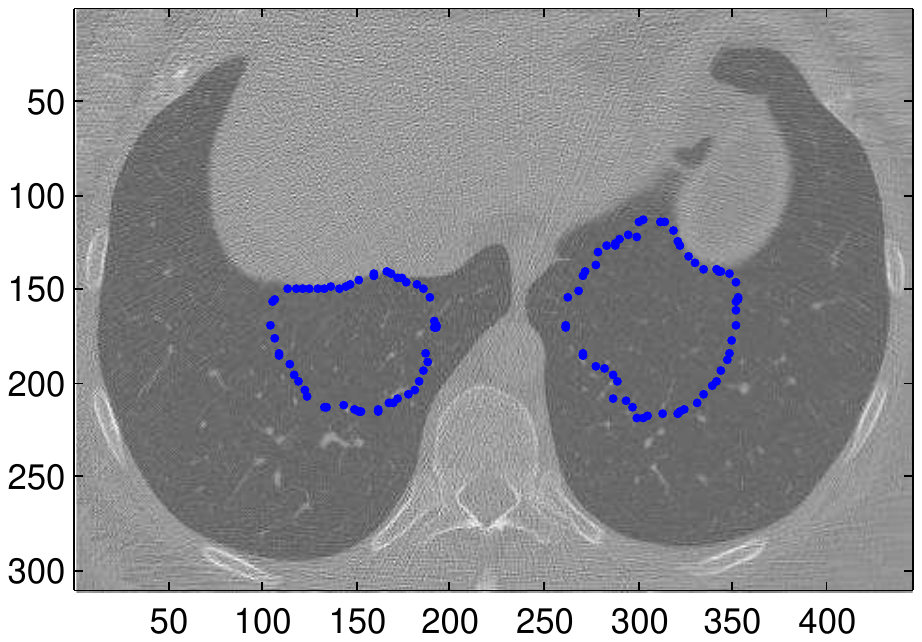}\\
		\includegraphics[viewport = 130 320 460 520, width = 0.15\textwidth]{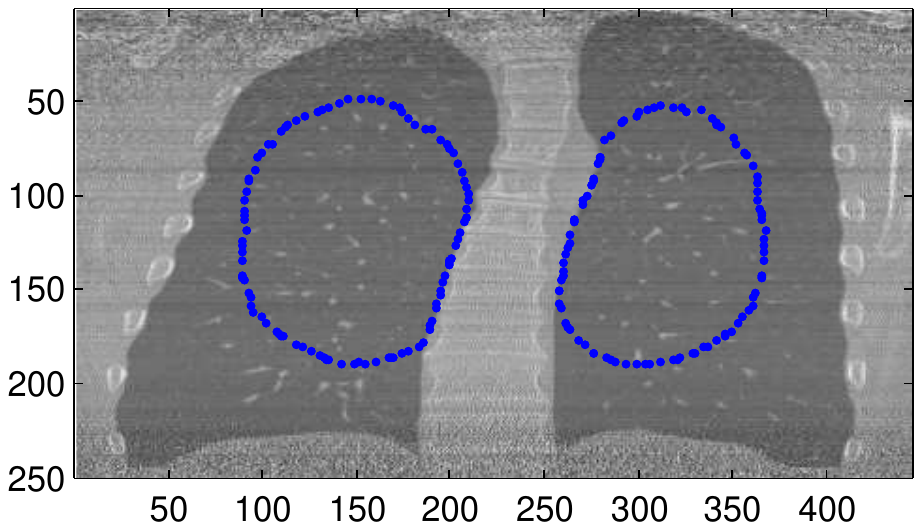}
		\includegraphics[viewport = 130 320 460 520, width = 0.15\textwidth]{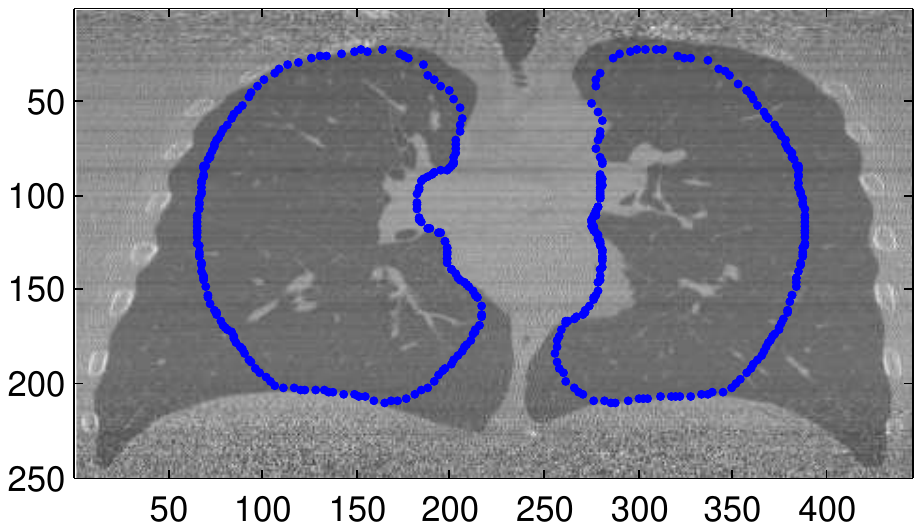} 
		\includegraphics[viewport = 130 320 460 520, width = 0.15\textwidth]{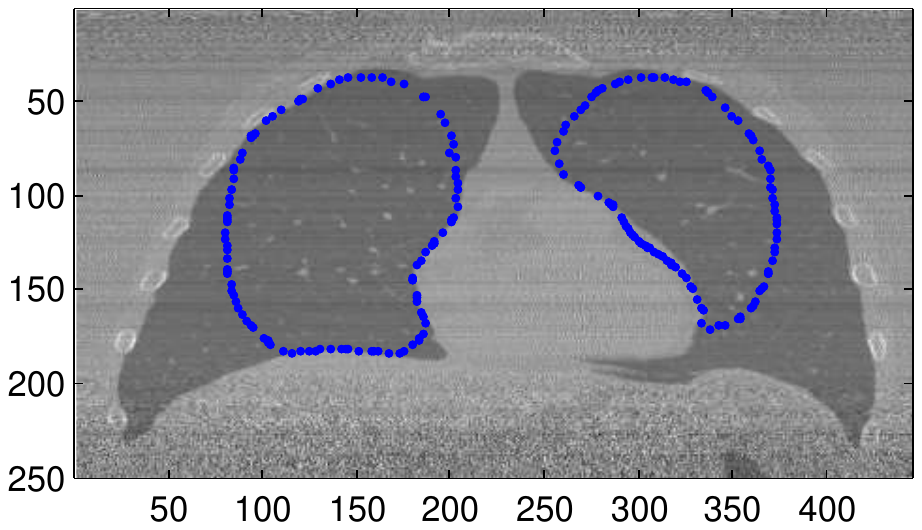} 
	\caption{Lung segmentation: Surfaces (row 1) and cross-sections (row 2: $z=80, 150, 200$, row 3: $y=80, 150, 200$) at $m=100$ at time $t=10$. The original images are from the Lung Image Database Consortium image collection (LIDC-IDRI) of The Cancer Imaging Archive (TCIA), see \cite{Reeves2007},\cite{Armato2011}, \cite{Reeves2011}.}
	\label{fig:lung_0100}
\end{figure}

%------------------------------------ m = 600 -----------------------------------------------------------------%
\begin{figure}
	\centering
		\includegraphics[width = 0.28\textwidth]{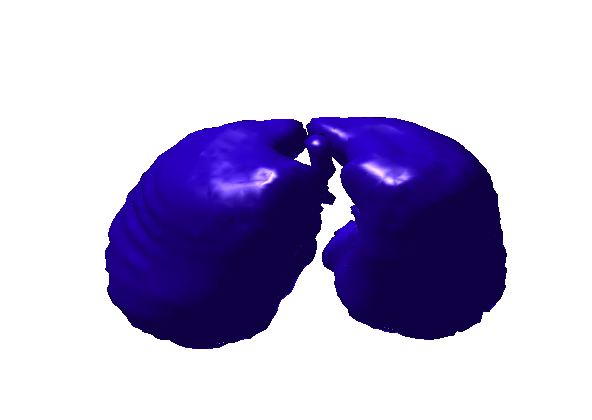}\\
		\includegraphics[viewport = 130 320 460 520, width = 0.15\textwidth]{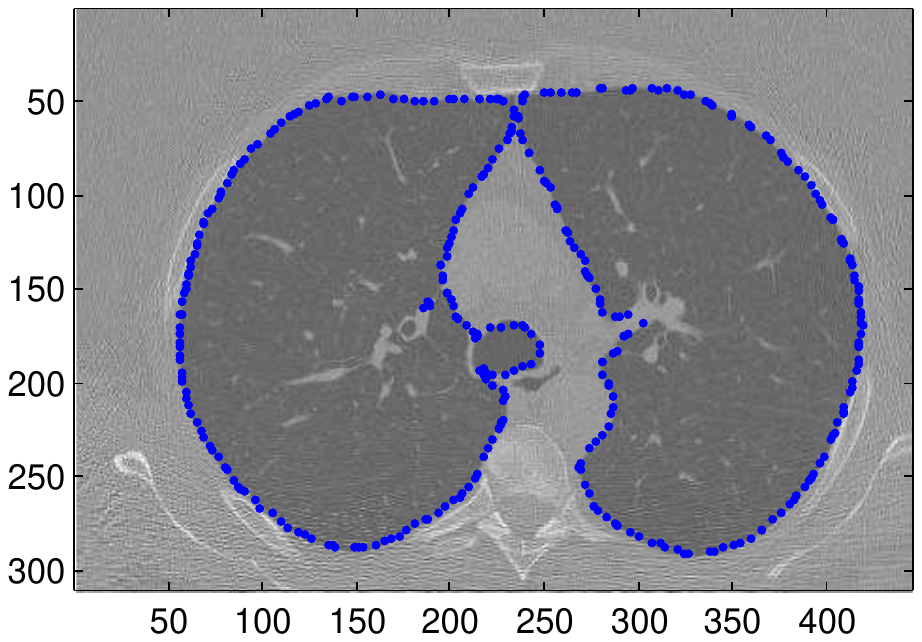}
		\includegraphics[viewport = 130 320 460 520, width = 0.15\textwidth]{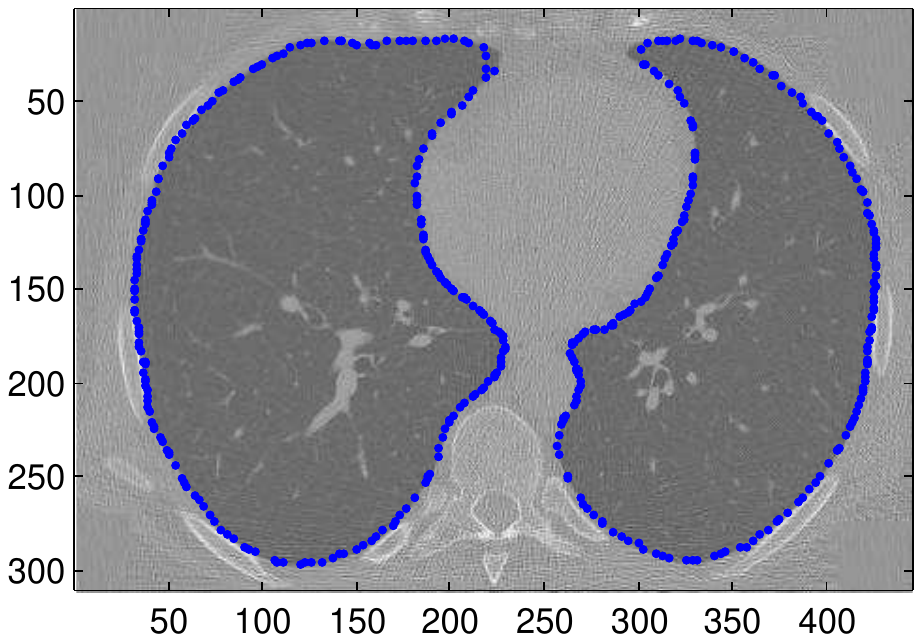}
		\includegraphics[viewport = 130 320 460 520, width = 0.15\textwidth]{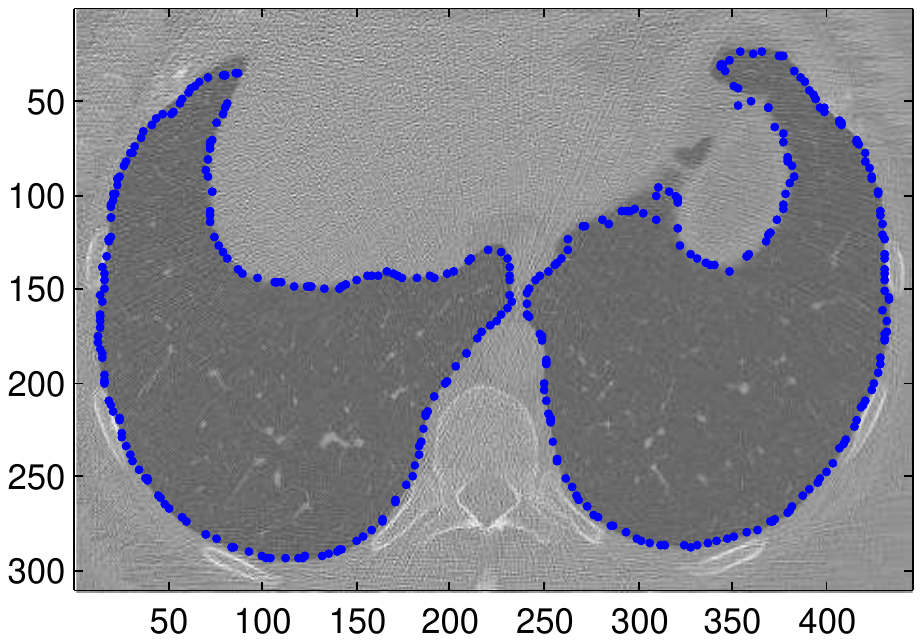}\\
		\includegraphics[viewport = 130 320 460 520, width = 0.15\textwidth]{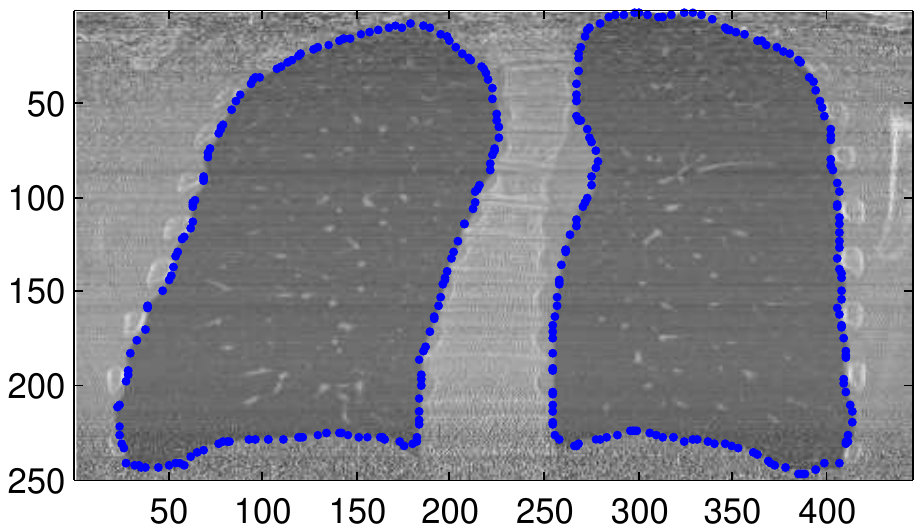}
		\includegraphics[viewport = 130 320 460 520, width = 0.15\textwidth]{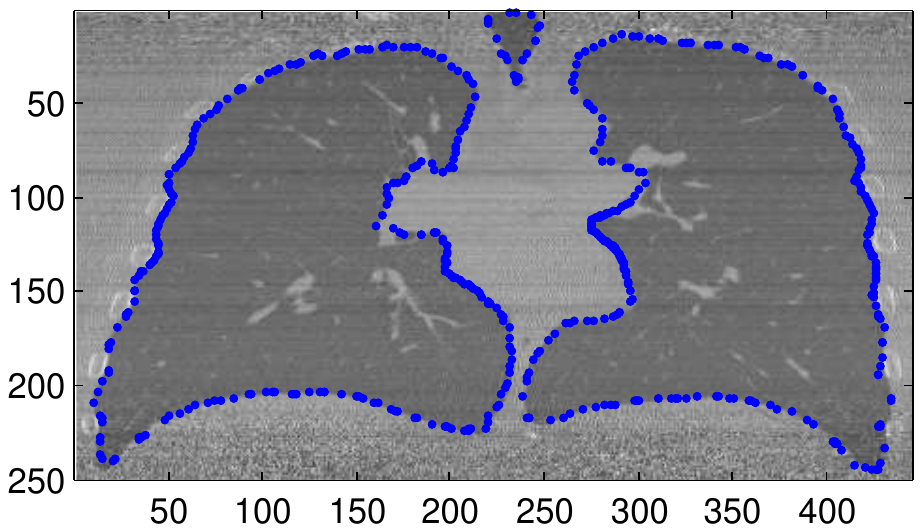} 
		\includegraphics[viewport = 130 320 460 520, width = 0.15\textwidth]{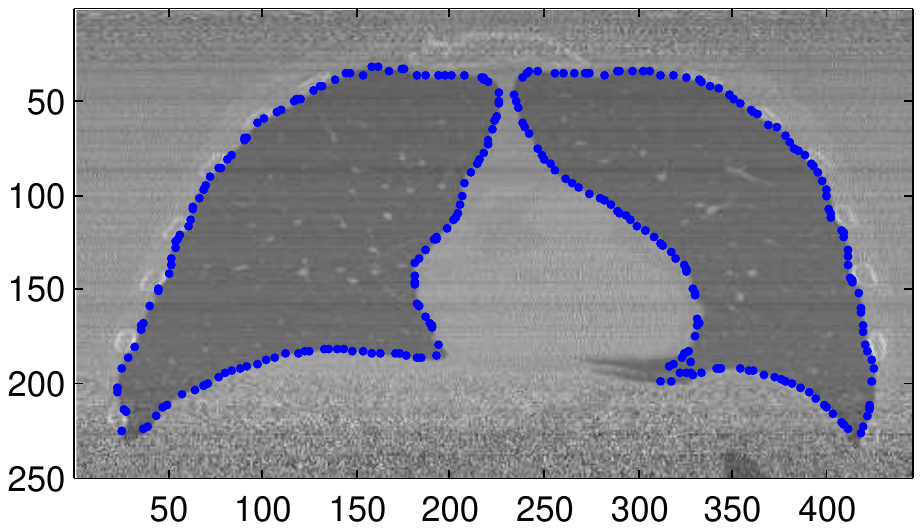} 
	\caption{Lung segmentation: Surfaces (row 1) and cross-sections (row 2: $z=80, 150, 200$, row 3: $y=80, 150, 200$) at $m=600$ at time $t=60$. The original images are from the Lung Image Database Consortium image collection (LIDC-IDRI) of The Cancer Imaging Archive (TCIA), see \cite{Reeves2007},\cite{Armato2011}, \cite{Reeves2011}.}
	\label{fig:lung_0600}
\end{figure}

In this section, we apply the segmentation method for three-dimensional images to medical image data. Segmentation of medical images is a challenging task due to possible high noise and image artifacts, see \cite{Sharma2010}. 

3D image data often consists of a set of 2D slice images generated by radiology scans, for example computed tomography (CT) and magnetic resonance (MR) scans. With a 3D image segmentation technique, one can segment organs (heart, lung, abdomen, liver, etc.) or tumors 
from their environment. The output, i.e. the resulting surface, serves as a reconstruction and visualization of the medical object and could be used for further medical analysis and diagnostic purposes: After the segmentation, one can compute the area of the triangulated surfaces and the volume of the enclosed regions. The area of the surfaces and the volume of the regions could be used for example to analyze if a tumor has been growing in the time between two radiological examinations.

First, we consider a sample 3D image of the Lung Image Database Consortium image collection (LIDC-IDRI) of The Cancer Imaging Archive (TCIA) (\url{https://wiki.cancerimagingarchive.net/display/Public/LIDC-IDRI}), see \cite{Reeves2007}, \cite{Armato2011}, \cite{Reeves2011})\footnote{The author acknowledges the National Cancer Institute and the Foundation for the National Institutes of Health, and their critical role in the creation of the free publicly available LIDC/IDRI Database.}. The data set consists of diagnostic CT scans. 
The original data set consists of 2D slice images stored as DICOM files.  The files are first preprocessed to cuboid 3D images with $N_x \times N_y \times N_z$ voxels, here: $N_x = 445$, $N_y = 310$ and $N_z = 250$. 

Figures \ref{fig:lung_0001}-\ref{fig:lung_0600} show the evolving 3D surfaces and six representative 2D cross-sections at different time steps $m=0, 100, 600$. 
In the subfigures showing 2D cross-sections, the image cross-sections for constant $z$ (in detail $z=80, 150, 200$) and constant $y$ (in detail $y=80, 150, 200$) are drawn as well as the intersection points of the surfaces' edges with the cross-section planes. 

For the image segmentation the weight of the curvature term is set to $\sigma=10$, the weight of the external forcing term to $\lambda = 1000$. As parameters for the time step size control, $\delta X_n^\mathrm{min} = 0.05$ and $\delta X_n^\mathrm{max} = 2$ are used. The time step size $\Delta t = 0.1$ need not be changed during the segmentation.

%-------------------------------------------------------------------------------------------------------------------------%
%---------------------------------------- UKR SPLITTING ------------------------------------------------------------------%
%-------------------------------------------------------------------------------------------------------------------------%

%-------------------------------------------------------------------------------------------------------------------------%
%---------------------------------------- m = 0 --------------------------------------------------------------------------%
%-------------------------------------------------------------------------------------------------------------------------%
\begin{figure}
	\centering
	  \includegraphics[trim = 20mm 10mm 20mm 10mm,clip, width = 0.3\textwidth]{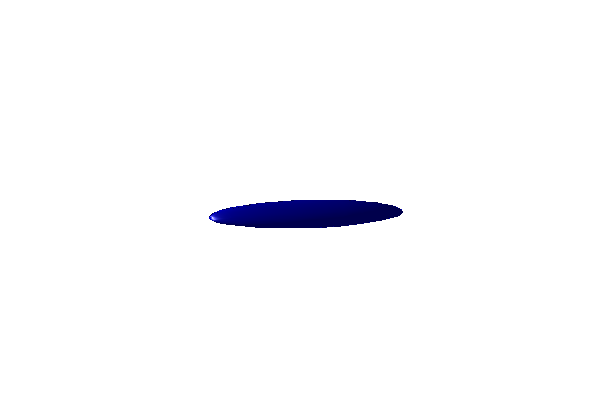}\\
		\includegraphics[viewport = 190 315 400 520, width = 0.15\textwidth]{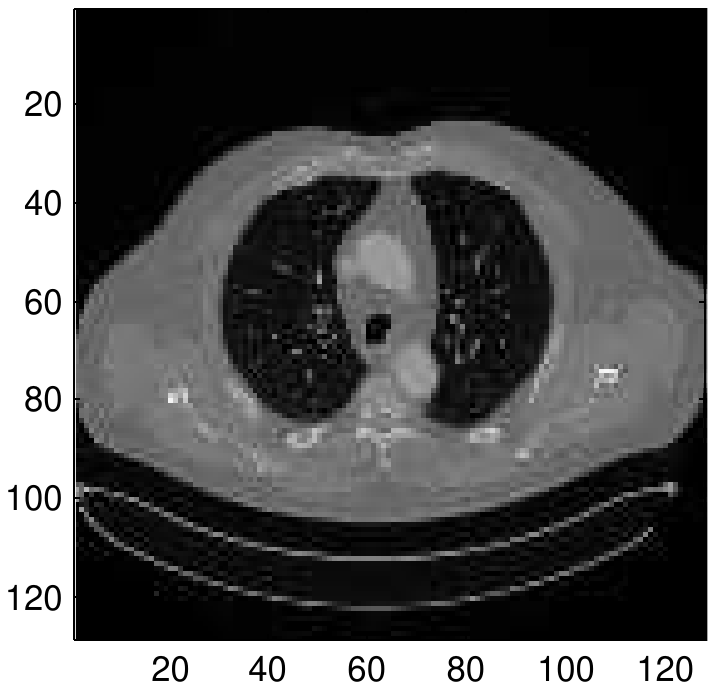}
		\includegraphics[viewport = 190 315 400 520, width = 0.15\textwidth]{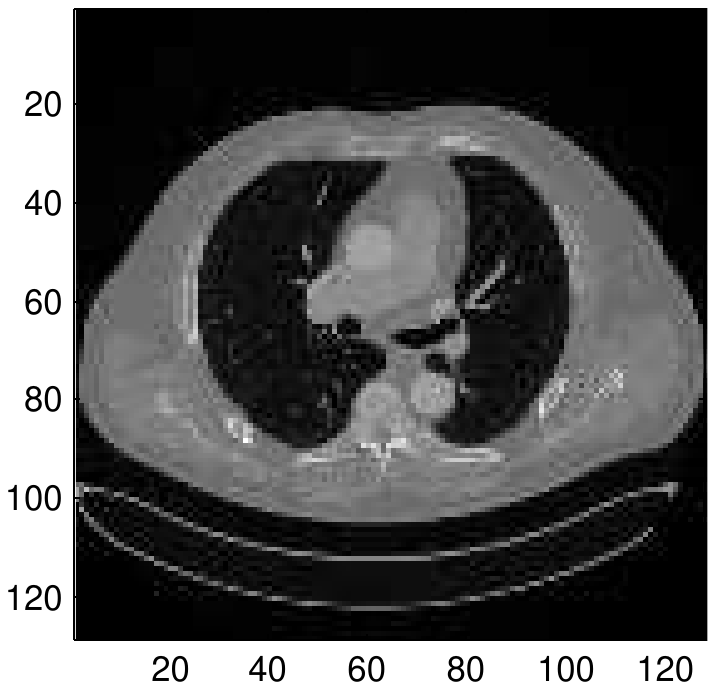}
		\includegraphics[viewport = 190 315 400 520, width = 0.15\textwidth]{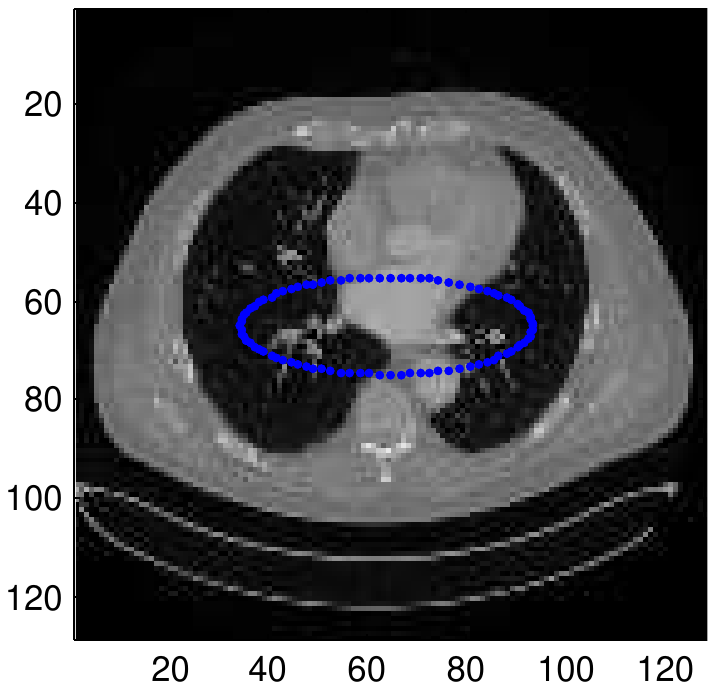}\\
		\includegraphics[viewport = 240 315 350 525, width = 0.1125\textwidth]{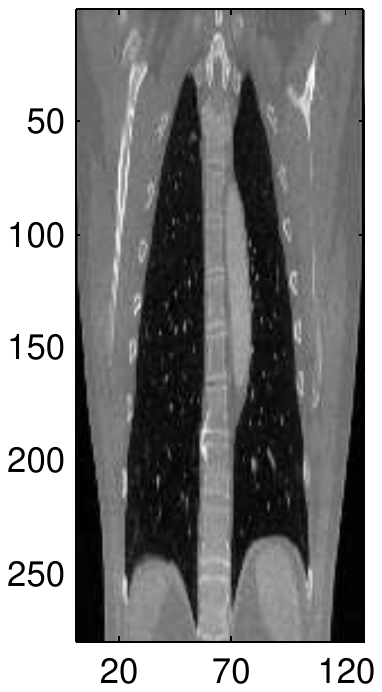}\hspace{4ex}
		\includegraphics[viewport = 240 315 350 525, width = 0.1125\textwidth]{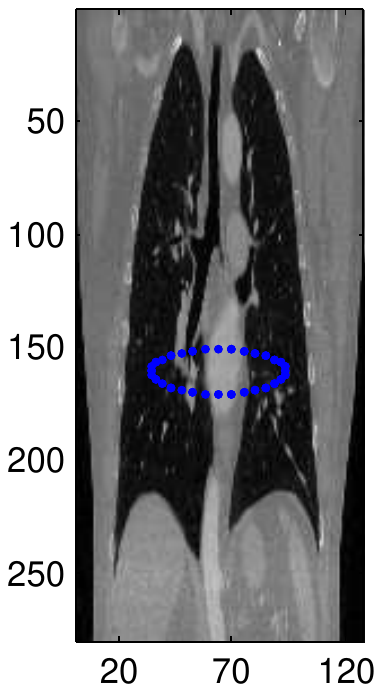}\hspace{4ex}
		\includegraphics[viewport = 240 315 350 525, width = 0.1125\textwidth]{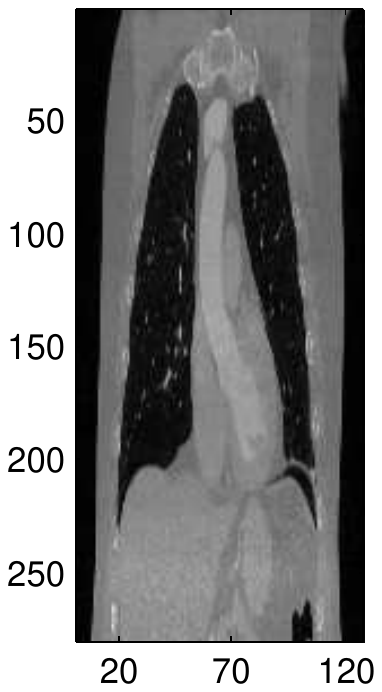}
	\caption{Lung segmentation with splitting: Surface at different viewing angles (row 1) and cross-sections (row 2: $z=80, 120, 160$, row 3: $y=50, 64, 80$) at $m=0$ at time $t=0$. Credits (original CT images): C. Stroszczynski, Radiology, University Hospital Regensburg.}
	\label{fig:cd4_part2_01_0000}
\end{figure}

%-------------------------------------------------------------------------------------------------------------------------%
%---------------------------------------- m = 50 --------------------------------------------------------------------------%
%-------------------------------------------------------------------------------------------------------------------------%
\begin{figure}
	\centering
	  \includegraphics[trim = 20mm 10mm 20mm 10mm,clip, width = 0.3\textwidth]{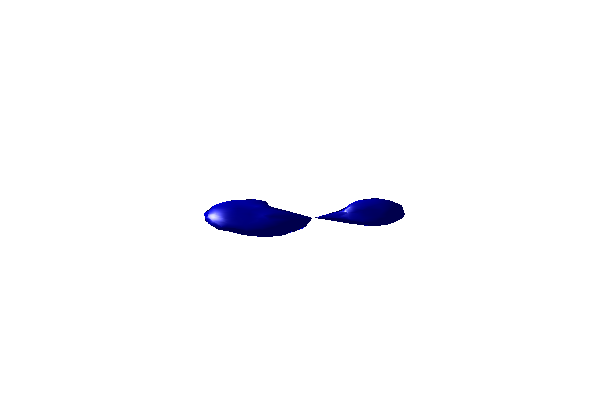}\\
		\includegraphics[viewport = 190 315 400 520, width = 0.15\textwidth]{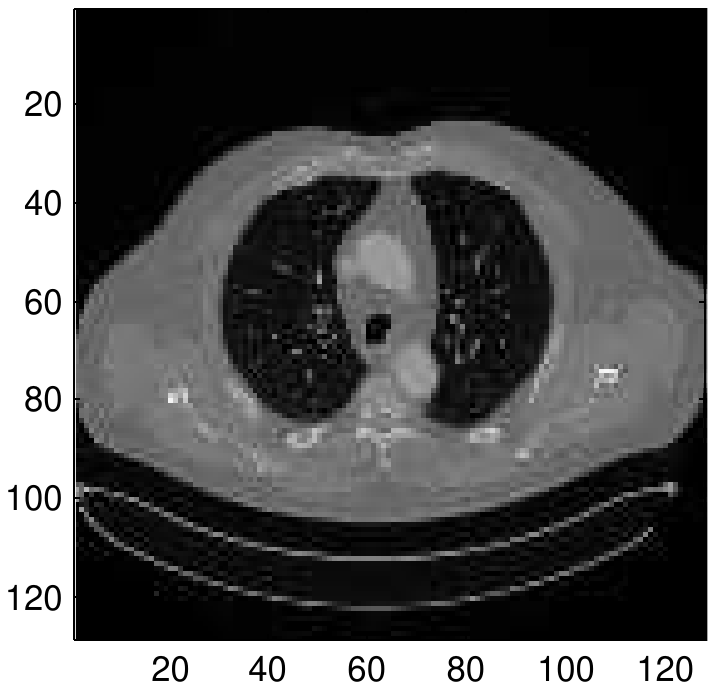}
		\includegraphics[viewport = 190 315 400 520, width = 0.15\textwidth]{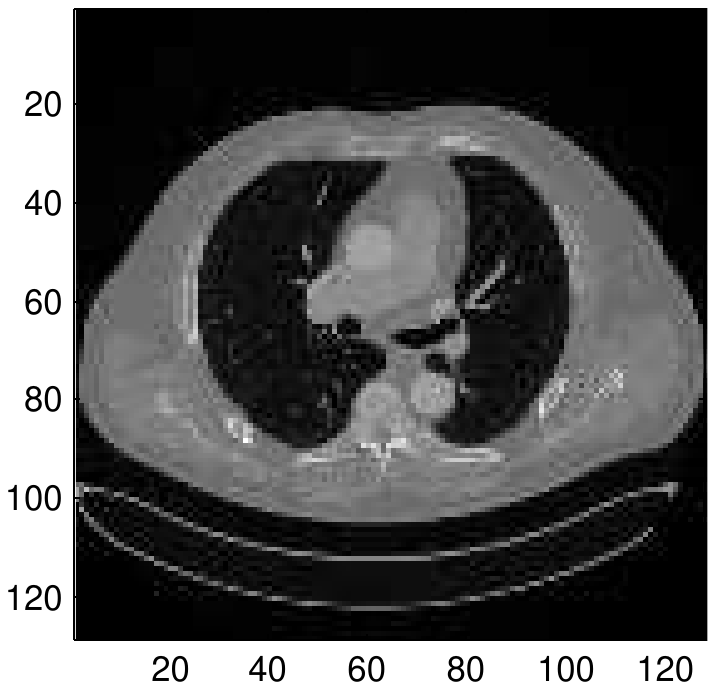}
		\includegraphics[viewport = 190 315 400 520, width = 0.15\textwidth]{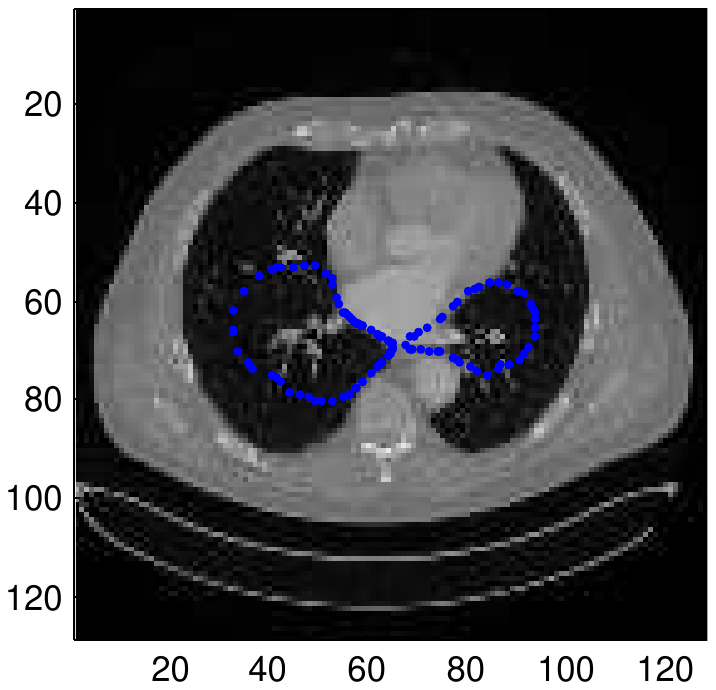}\\
		\includegraphics[viewport = 240 315 350 525, width = 0.1125\textwidth]{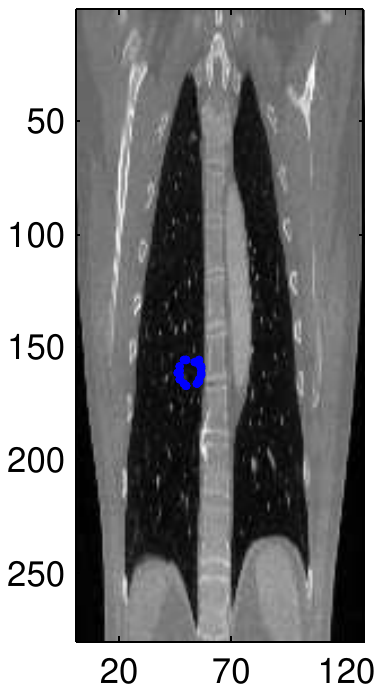}\hspace{4ex}
		\includegraphics[viewport = 240 315 350 525, width = 0.1125\textwidth]{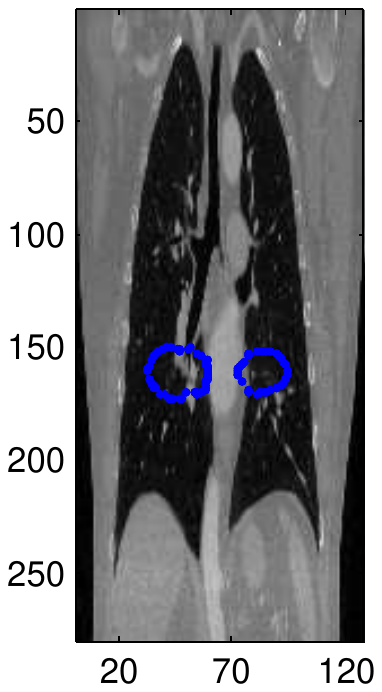}\hspace{4ex}
		\includegraphics[viewport = 240 315 350 525, width = 0.1125\textwidth]{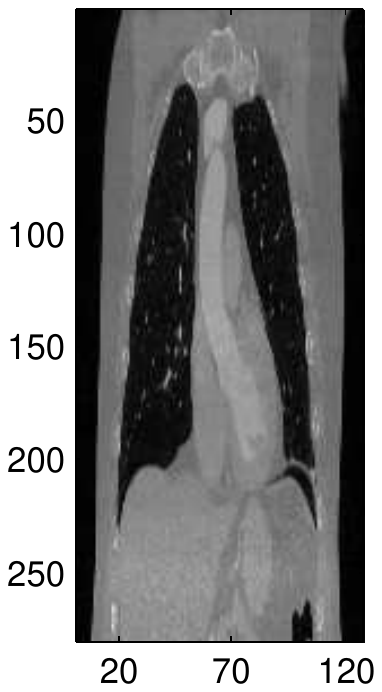}
	\caption{Lung segmentation with splitting: Surfaces  (row 1) and cross-sections (row 2: $z=80, 120, 160$, row 3: $y=50, 64, 80$) at $m=50$ at time $t=10$. Credits (original CT images): C. Stroszczynski, Radiology, University Hospital Regensburg.}
	\label{fig:cd4_part2_01_0050}
\end{figure}

%-------------------------------------------------------------------------------------------------------------------------%
%---------------------------------------- m = 100 --------------------------------------------------------------------------%
%-------------------------------------------------------------------------------------------------------------------------%
\begin{figure}
	\centering
	  \includegraphics[trim = 20mm 10mm 20mm 10mm,clip, width = 0.3\textwidth]{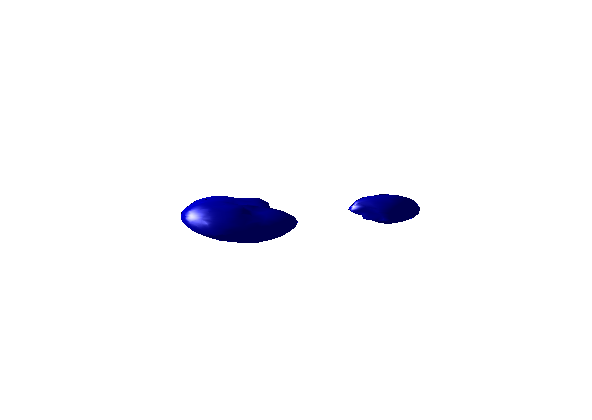}\\
		\includegraphics[viewport = 190 315 400 520, width = 0.15\textwidth]{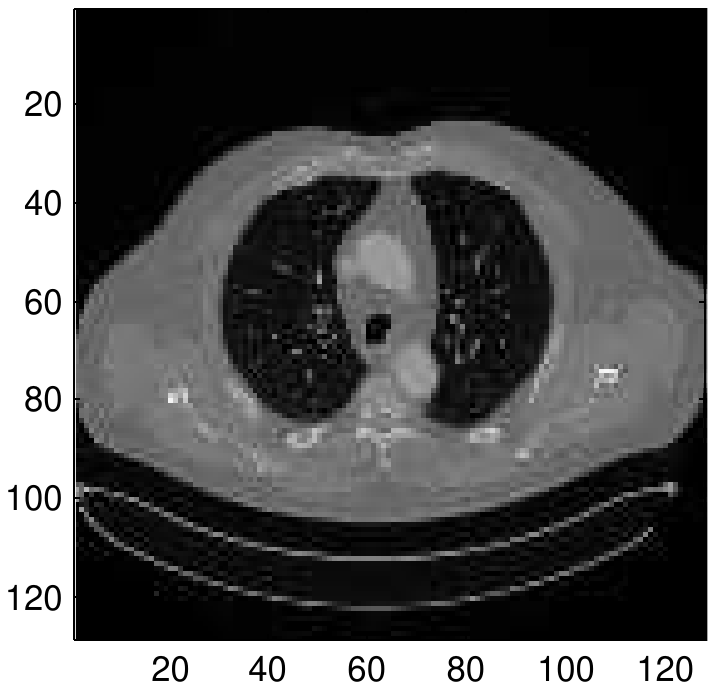}
		\includegraphics[viewport = 190 315 400 520, width = 0.15\textwidth]{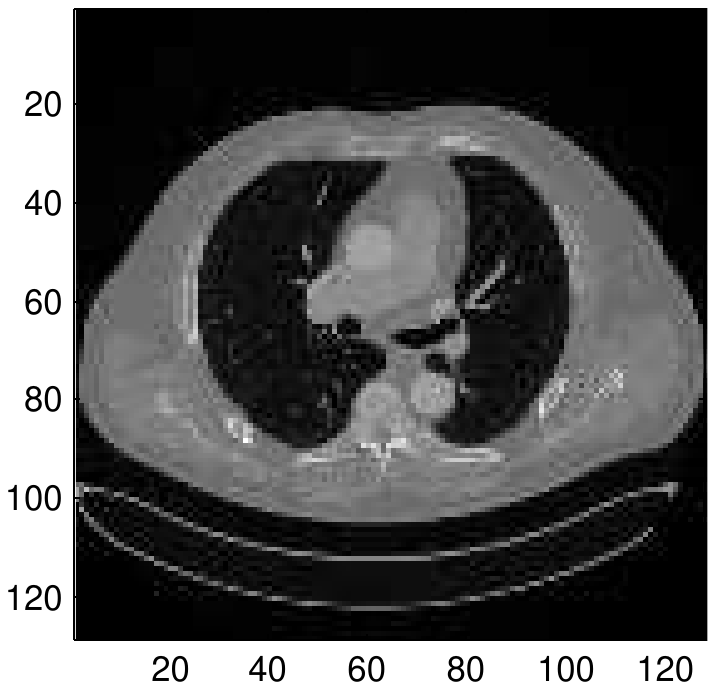}
		\includegraphics[viewport = 190 315 400 520, width = 0.15\textwidth]{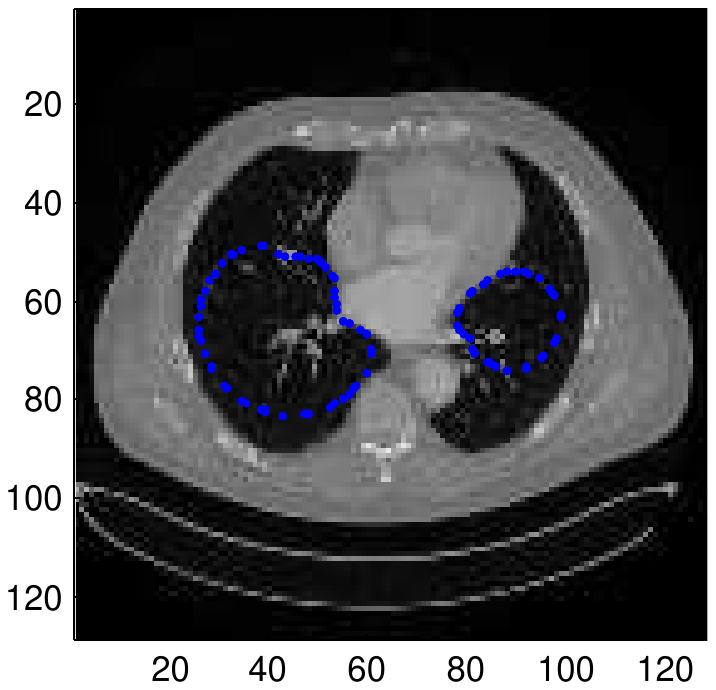}\\
		\includegraphics[viewport = 240 315 350 525, width = 0.1125\textwidth]{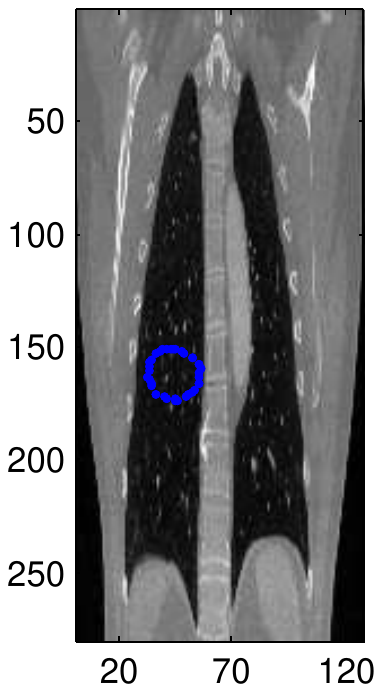}\hspace{4ex}
		\includegraphics[viewport = 240 315 350 525, width = 0.1125\textwidth]{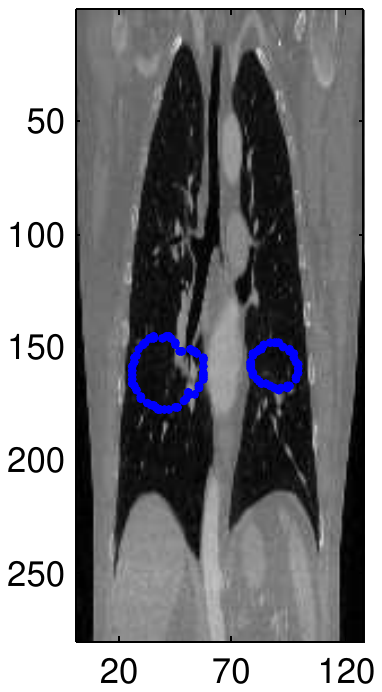}\hspace{4ex}
		\includegraphics[viewport = 240 315 350 525, width = 0.1125\textwidth]{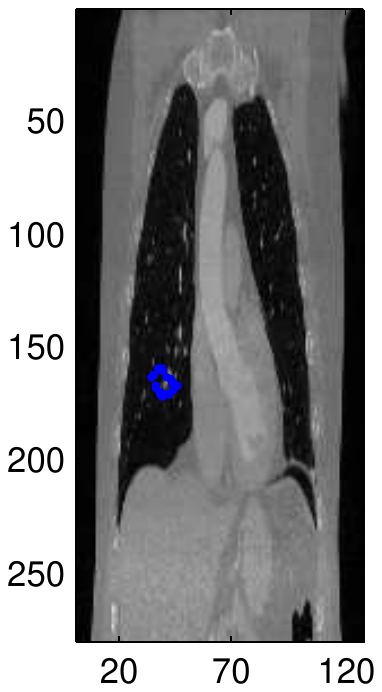}
	\caption{Lung segmentation with splitting: Surfaces (row 1) and cross-sections (row 2: $z=80, 120, 160$, row 3: $y=50, 64, 80$) at $m=100$ at time $t=20$. Credits (original CT images): C. Stroszczynski, Radiology, University Hospital Regensburg.}
	\label{fig:cd4_part2_01_0100}
\end{figure}

%-------------------------------------------------------------------------------------------------------------------------%
%---------------------------------------- m = 500 --------------------------------------------------------------------------%
%-------------------------------------------------------------------------------------------------------------------------%
\begin{figure}
	\centering
	  \includegraphics[trim = 20mm 10mm 20mm 10mm,clip, width = 0.3\textwidth]{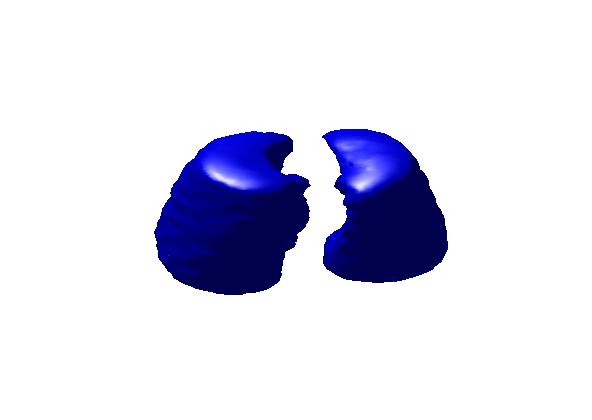}\\
		\includegraphics[viewport = 190 315 400 520, width = 0.15\textwidth]{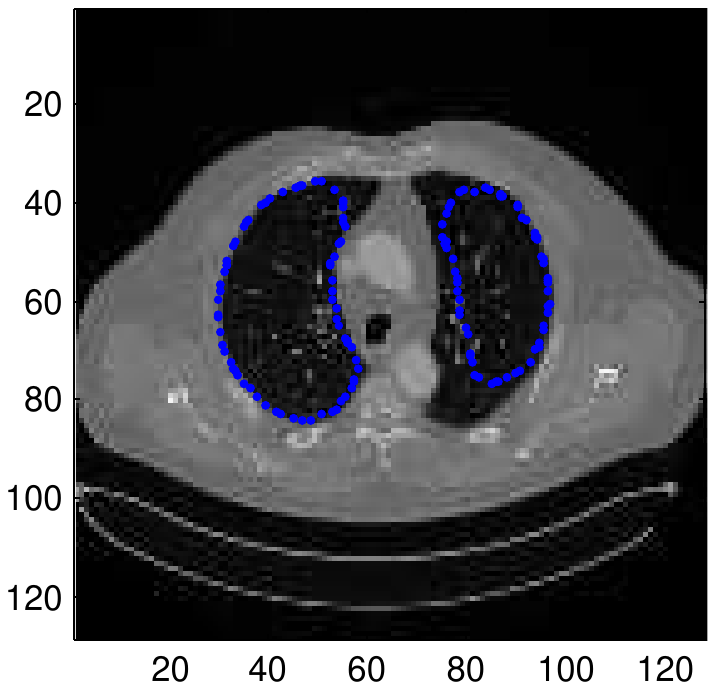}
		\includegraphics[viewport = 190 315 400 520, width = 0.15\textwidth]{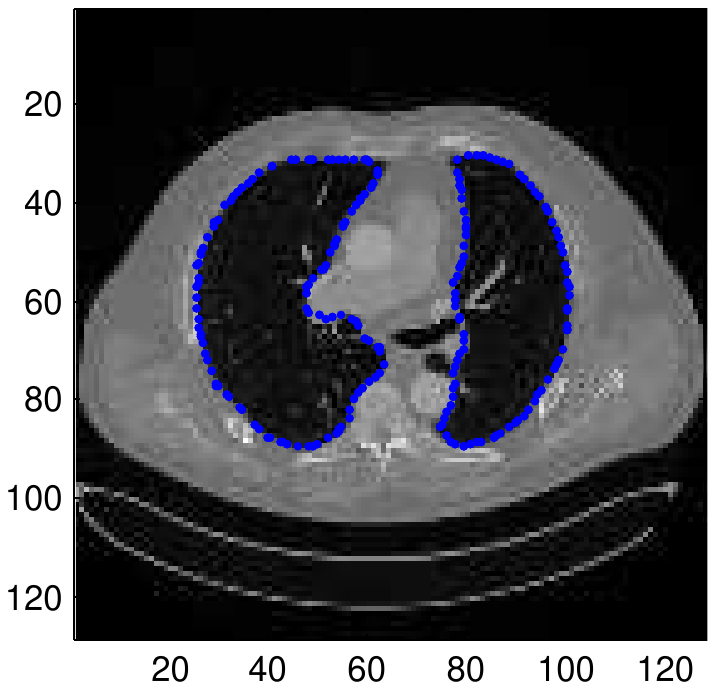}
		\includegraphics[viewport = 190 315 400 520, width = 0.15\textwidth]{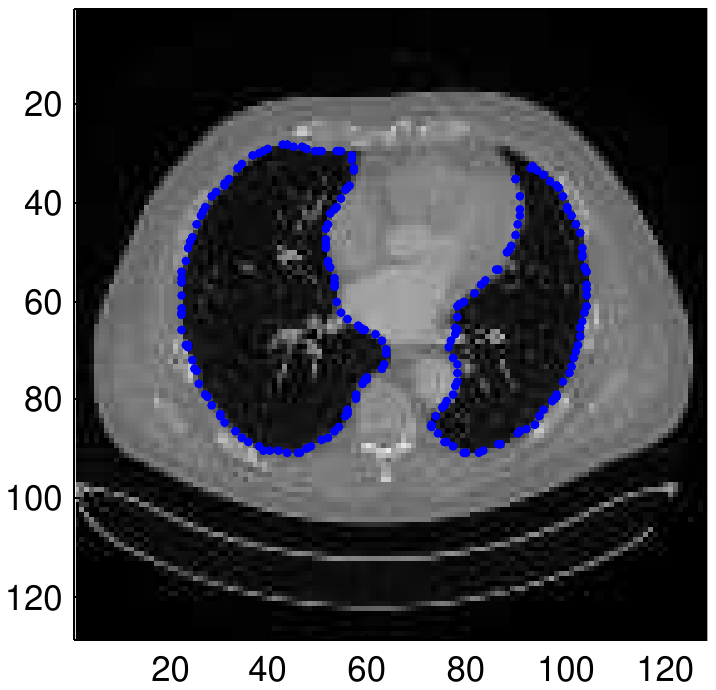}\\
		\includegraphics[viewport = 240 315 350 525, width = 0.1125\textwidth]{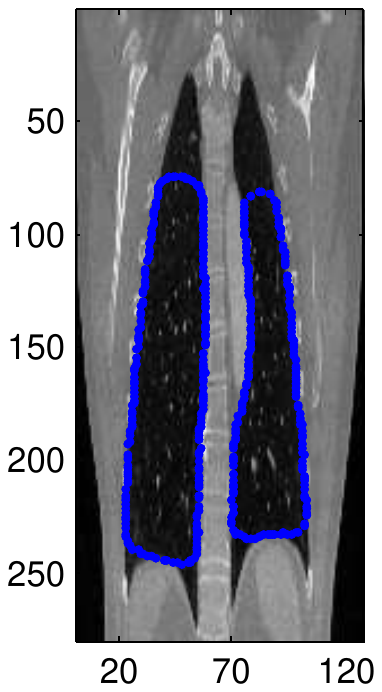}\hspace{4ex}
		\includegraphics[viewport = 240 315 350 525, width = 0.1125\textwidth]{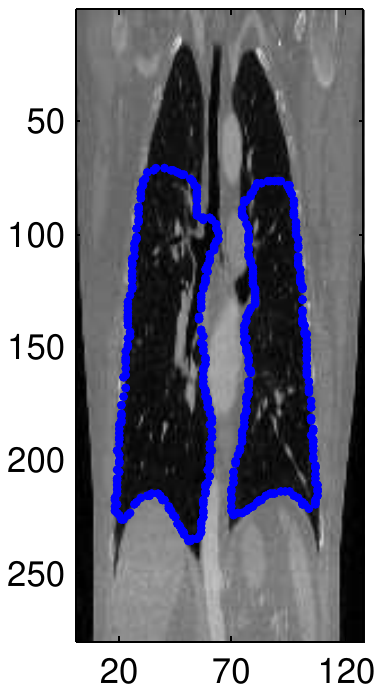}\hspace{4ex}
		\includegraphics[viewport = 240 315 350 525, width = 0.1125\textwidth]{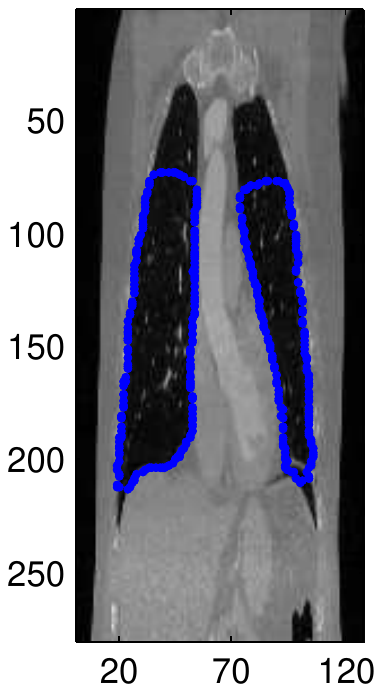}
	\caption{Lung segmentation with splitting: Surfaces  (row 1) and cross-sections (row 2: $z=80, 120, 160$, row 3: $y=50, 64, 80$) at $m=500$ at time $t=100$. Credits (original CT images): C. Stroszczynski, Radiology, University Hospital Regensburg.}
	\label{fig:cd4_part2_01_0500}
\end{figure}

%-------------------------------------------------------------------------------------------------------------------------%
%---------------------------------------- m = 900 --------------------------------------------------------------------------%
%-------------------------------------------------------------------------------------------------------------------------%
\begin{figure}
	\centering
	  \includegraphics[trim = 20mm 10mm 20mm 10mm,clip, width = 0.3\textwidth]{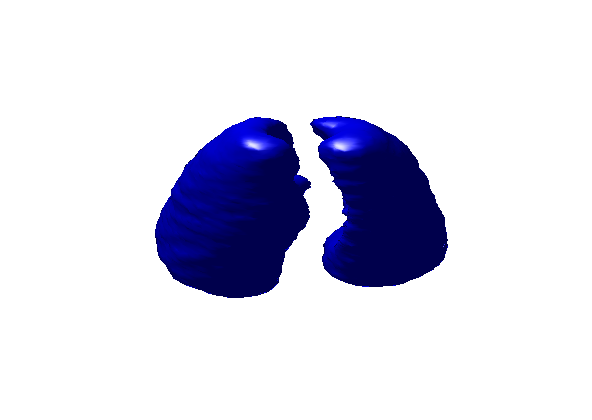}\\
		\includegraphics[viewport = 190 315 400 520, width = 0.15\textwidth]{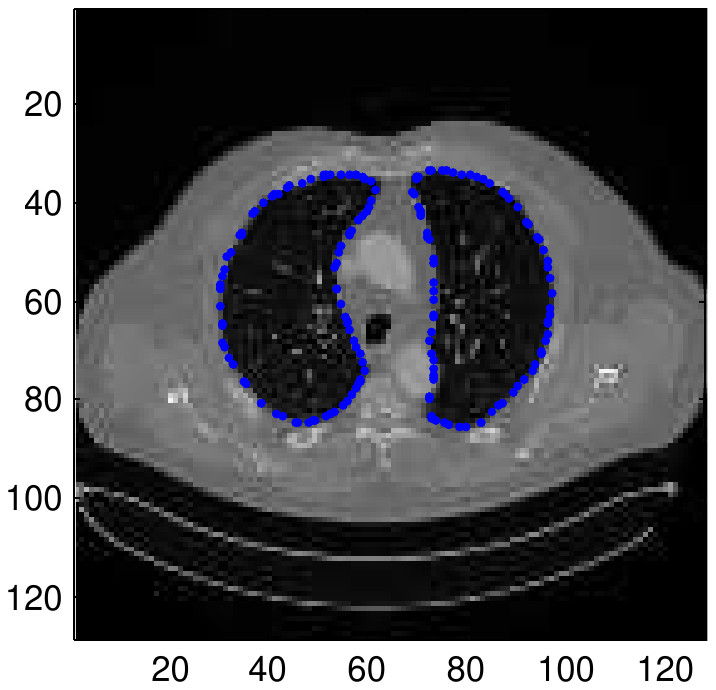}
		\includegraphics[viewport = 190 315 400 520, width = 0.15\textwidth]{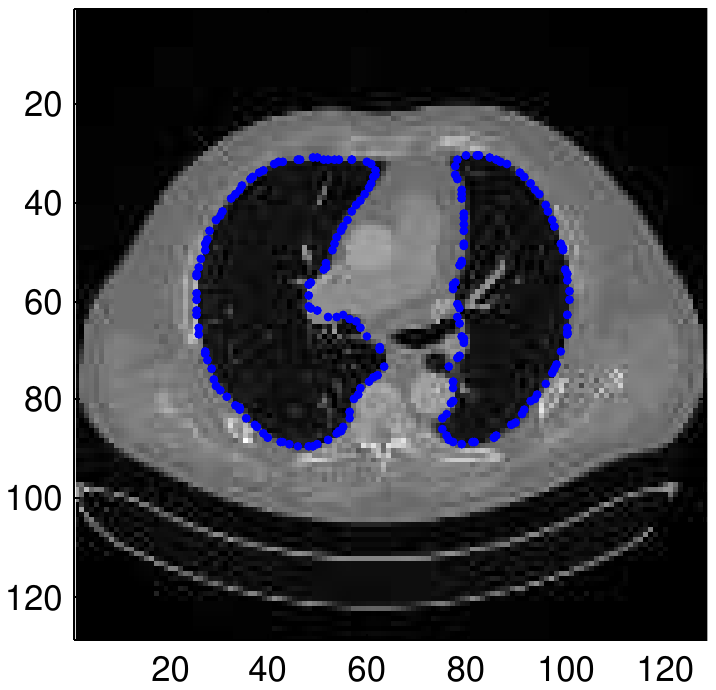}
		\includegraphics[viewport = 190 315 400 520, width = 0.15\textwidth]{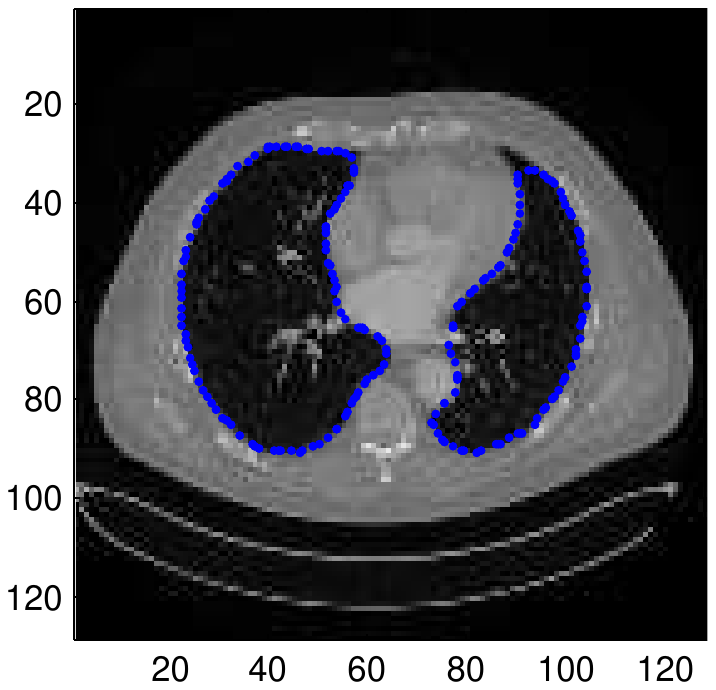}\\
		\includegraphics[viewport = 240 315 350 525, width = 0.1125\textwidth]{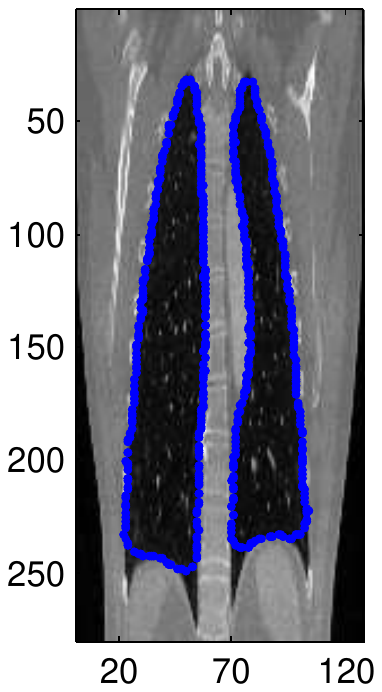}\hspace{4ex}
		\includegraphics[viewport = 240 315 350 525, width = 0.1125\textwidth]{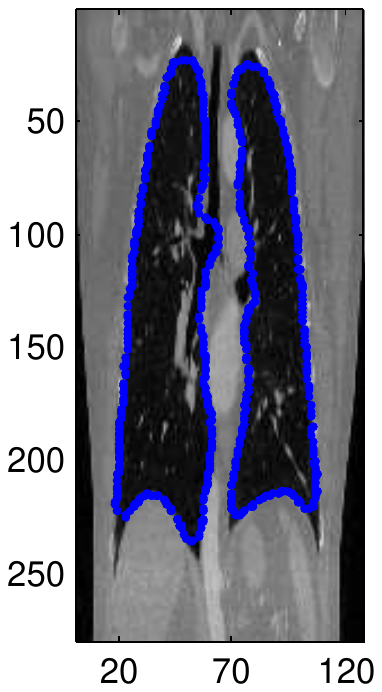}\hspace{4ex}
		\includegraphics[viewport = 240 315 350 525, width = 0.1125\textwidth]{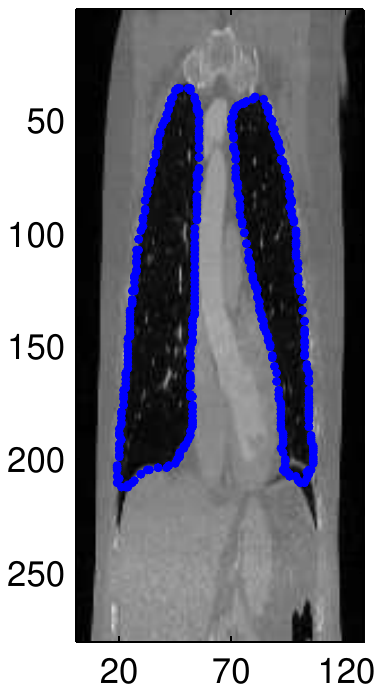}
	\caption{Lung segmentation with splitting: Surfaces (row 1) and cross-sections (row 2: $z=80, 120, 160$, row 3: $y=50, 64, 80$) at $m=900$ at time $t=180$. Credits (original CT images): C. Stroszczynski, Radiology, University Hospital Regensburg.}
	\label{fig:cd4_part2_01_0900}
\end{figure}

In a second experiment, we consider an experiment where a topology change occurs. We perform again a lung segmentation starting now with one initial surface which is split into two surfaces. Figure~\ref{fig:cd4_part2_01_0000}-\ref{fig:cd4_part2_01_0900} show the surface(s) at time step $m=0, 50, 100,500, 900$ as well as cross-sections of the image and of the surface(s). For the cross-sections, we consider the planes given by $z=80, 120, 160$ and $y=50, 64, 80$. 

The splitting occurs at time step $m=50$. To detect the topology change, we use an auxiliary background grid with grid size $a=2$. A cube of the grid is considered for possible topology changes if more than $N_\mathrm{detect}=8$ nodes are located inside the cube. Further, we use the parameters $thr1=30^\circ$, $thr2=150^\circ$ and  $thr3=40^\circ$, recall Section~\ref{subsec:identification_top_change_3D}. After the splitting, the two surfaces grow and new triangles are created by bisection of too large triangles. For the segmentation we use the parameters $\sigma = 1$ and $\lambda = 20$. The time step size is set to $\Delta t = 0.2$ with time step control using $\delta X_{n}^{\mathrm{max}} = 2$, $\delta X_{n}^{ \mathrm{min}} = 0.1$. However, no increase or decrease of the time step size is necessary. 

As postprocessing step, we compute the volume of the two enclosed regions and the area of the region boundaries. The right lung of the patient, i.e. the left surface in the Figure~\ref{fig:cd4_part2_01_0900}, has an area of $A_1 = 3.309\cdot 10^4$ and a volume of $V_1 = 2.691\cdot 10^5$ (CT images are mirror images). The left lung of the patient (right surface in the figure) has an area of $A_2 = 2.801\cdot 10^4$ and a volume of $V_2 = 1.923\cdot 10^5$. Thus, as expected, the volume of the right lung is larger compared to the left lung. Note, that we handle a voxel as a cube with side length 1, resulting in values of magnitude $10^4$ for the area and $10^5$ for the volume. If details on the acquisition system of the CT images are known (like the slice thickness, and the height and width of one pixel of a slice image), the area and the volume can be computed precisely and can be expressed in the metric system for practical interpretation of the values.  

%% file: 5_conclusion.tex
\section{Conclusion}
\label{sec:conclusion}
We presented a new parametric method for segmentation of 3D images. We considered extensions of the Mumford-Shah and Chan-Vese functional for 3D image segmentation by active surface. For the time-dependent surfaces, we proposed a parametric scheme and introduced an efficient numerical scheme based on a finite element approximation. A novel method to detect and perform topology changes of the surfaces has been presented which uses a virtual auxiliary background grid. Due to the fact that for the main computations only a two-dimensional grid is used, the developed method is very efficient from a computational point of view. Several artificial images have been studied to demonstrate splitting and merging of surfaces, and increase and decrease of the genus of a surface. We successfully applied our method to real medical 3D image data from computed tomography, including an example with a topology change.